\documentclass[journal,twoside]{IEEEtran}

\usepackage{cite}
\usepackage{textcomp}
\usepackage{graphicx}
\usepackage{color}

\usepackage{amsthm}

\usepackage{marginnote}
\usepackage{etoolbox}
\makeatletter
\@ifundefined{color@begingroup}%
  {\let\color@begingroup\relax
   \let\color@endgroup\relax}{}%
\def\fix@ieeecolor@hbox#1{%
  \hbox{\color@begingroup#1\color@endgroup}}
\patchcmd\@makecaption{\hbox}{\fix@ieeecolor@hbox}{}{\FAILED}
\patchcmd\@makecaption{\hbox}{\fix@ieeecolor@hbox}{}{\FAILED}
\usepackage[utf8]{inputenc}
\usepackage{xspace}

\usepackage{amsmath,amssymb}
\usepackage{bm,fixmath} 
\usepackage{accents}
\usepackage{pifont}

\usepackage{cleveref}

\usepackage{makecell}

\usepackage{graphicx}

\usepackage{enumitem}
\setlist[description]{style=nextline}

\usepackage{algorithm}
\usepackage{algpseudocode}
\algrenewcommand\algorithmicrequire{\textbf{Input:}}
\newcommand{\CommentState}[1]{\Statex\hspace{\algorithmicindent}{\color{blue}// #1}}

\newlength{\continueindent}
\usepackage{etoolbox}
\makeatletter
\newcommand*{\ALG@customparshape}{\parshape 2 \leftmargin \linewidth \dimexpr\ALG@tlm+\continueindent\relax \dimexpr\linewidth+\leftmargin-\ALG@tlm-\continueindent\relax}
\apptocmd{\ALG@beginblock}{\ALG@customparshape}{}{\errmessage{failed to patch}}
\makeatother

\DeclareMathOperator*{\argmin}{arg\,min}
\DeclareMathOperator{\prox}{prox}
\DeclareMathOperator{\refl}{refl}
\DeclareMathOperator{\proj}{proj}

\newtheorem{definition}{Definition}
\newtheorem{remark}{Remark}

\newtheorem{proposition}{Proposition}

\newtheorem{assumption}{Assumption}
\newtheorem{lemma}{Lemma}
\newtheorem{corollary}{Corollary}

\crefname{appsec}{Appendix}{Appendices}
\crefname{assumption}{Assumption}{Assumptions}
\crefname{corollary}{Corollary}{Corollaries}
\crefname{algorithm}{Algorithm}{Algorithms}
\crefname{proposition}{Proposition}{Propositions}
\crefname{lemma}{Lemma}{Lemmas}
\crefname{definition}{Definition}{Definitions}
\crefname{remark}{Remark}{Remarks}
\crefname{figure}{Figure}{Figures}
\crefname{table}{Table}{Tables}
\crefformat{equation}{(#2#1#3)}
\crefrangeformat{equation}{(#3#1#4) to~(#5#2#6)}
\crefmultiformat{equation}{(#2#1#3)}{ and~(#2#1#3)}{, (#2#1#3)}{ and~(#2#1#3)}

\newcommand{\norm}[1]{\left\lVert#1\right\rVert}
\newcommand{\ev}[1]{\mathbb{E}\left[#1\right]}

\newcommand{\N}{\mathbb{N}}
\newcommand{\R}{\mathbb{R}}
\newcommand{\C}{\mathbb{C}}

\newcommand{\e}{\mathbold{e}}
\newcommand{\tv}{\mathbold{t}}

\newcommand{\x}{\mathbold{x}}
\newcommand{\y}{\mathbold{y}}
\newcommand{\vv}{\mathbold{v}}
\newcommand{\w}{\mathbold{w}}
\newcommand{\z}{\mathbold{z}}
\renewcommand{\Im}{\mathbold{I}}

\newcommand{\1}{\pmb{1}}
\newcommand{\0}{\boldsymbol{0}}

\def\lmin{{\underline{\lambda}}}
\def\lmax{{\bar{\lambda}}}

\newcommand{\I}{\mathcal{I}}
\newcommand{\T}{\mathcal{T}}
\newcommand{\X}{\mathcal{X}}

\newcommand{\Ne}{{N_\mathrm{e}}} 

\newcommand{\scs}[3]{\mathcal{F}_{#1, #2}\left( #3 \right)} 
\newcommand{\ccp}[1]{\scs{0}{\infty}{#1}} 

\newcommand{\ubar}[1]{\underaccent{\bar}{#1}}
\newcommand{\pmin}{\ubar{p}}
\newcommand{\pmax}{\bar{p}}
\algrenewcommand\algorithmicensure{\textbf{Output:}}

\newcommand{\cmark}{\ding{51}\xspace}
\newcommand{\xmark}{\ding{55}\xspace}

\newcommand{\alg}{{Fed-PLT}\xspace}

\newcommand{\tg}{t_\mathrm{G}}
\newcommand{\tc}{t_\mathrm{C}}

\title{Enhancing Privacy in Federated Learning\\through Local Training}

\author{Nicola Bastianello, Changxin Liu, Karl H. Johansson
\thanks{This work was partially supported by the European Union’s Horizon Research and Innovation Actions programme under grant agreement No. 101070162, partially by Swedish Research Council Distinguished Professor Grant 2017-01078 Knut and Alice Wallenberg Foundation Wallenberg Scholar Grant, and partially by the NSERC Postdoctoral Fellowship PDF-577876-2023.}
\thanks{N. Bastianello and K. H. Johansson are with the School of Electrical Engineering and Computer Science, and Digital Futures, KTH Royal Institute of Technology, Sweden.}
\thanks{C. Liu is with the Key Laboratory of Smart Manufacturing in Energy Chemical Process, Ministry of Education, East China University of Science and Technology, Shanghai 200237, China.}
\thanks{Corresponding author: Nicola Bastianello {\tt\small nicolba@kth.se}}
}

\begin{document}

\maketitle

\begin{abstract}
In this paper we propose the federated learning algorithm Fed-PLT to overcome the challenges of \textit{(i)} expensive communications and \textit{(ii)} privacy preservation.
We address \textit{(i)} by allowing for both partial participation and local training, which significantly reduce the number of communication rounds between the central coordinator and computing agents.
The algorithm matches the state of the art in the sense that the use of local training demonstrably does not impact accuracy. Additionally, agents have the flexibility to choose from various local training solvers, such as (stochastic) gradient descent and accelerated gradient descent.
Further, we investigate how employing local training can enhance privacy, addressing point \textit{(ii)}. In particular, we derive differential privacy bounds and highlight their dependence on the number of local training epochs.
We assess the effectiveness of the proposed algorithm by comparing it to alternative techniques, considering both theoretical analysis and numerical results from a classification task.
\end{abstract}

\begin{IEEEkeywords}
Federated learning, privacy, local training, partial participation
\end{IEEEkeywords}


\section{Introduction}\label{sec:introduction}
Federated learning has proved to be a powerful framework to enable private cooperative learning \cite{li_federated_2020,zhang_survey_2021,gafni_federated_2022}. Indeed, federated learning allows a set of agents to pool their resources together to train a more accurate model than any single agent could. This is accomplished by relying on a central coordinator that iteratively collects and aggregates locally trained models -- without the need to share raw data.
Federated learning has been successfully deployed in different applications, including healthcare \cite{rauniyar_federated_2023}, wireless networks \cite{qian_distributed_2022}, mobile devices, Internet-of-Things, and sensor networks \cite{li_review_2020}, power systems \cite{cheng_review_2022}, and intelligent transportation \cite{zhang_federated_2024}, to name a few.

Clearly, all these applications may involve potentially sensitive and/or proprietary training data. Therefore, a fundamental challenge in the federated set-up is to \textit{guarantee the privacy of these data} throughout the learning process.
However, despite the fact that in a federated architecture the agents are not required to share raw data, various attacks have emerged to extract private information from the trained model or its predictions \cite{nasr2019comprehensive,lyu_privacy_2024,liu_survey_2024}. Thus, employing the federated learning architecture is not sufficient to ensure privacy, and tailored techniques need to be developed \cite{girgis2021shuffled}.

Besides privacy preservation, another fundamental challenge in federated learning is the \textit{expensiveness of communications}.
%
Indeed, the models exchanged by the agents are oftentimes high dimensional, especially when training (deep) neural networks \cite{gafni_federated_2022}, resulting in resource-intensive communications.
Thus, one of the most important objectives is to design federated learning algorithms that employ as few communications as possible, while still guaranteeing good accuracy of the trained model.
Different heuristics have been proposed to accomplish this objective, with the main ones arguably being \textit{partial participation}\footnote{A.k.a. \textit{client selection}.} \cite{nemeth_snapshot_2022}, \textit{local training} \cite{grudzien_can_2023}, and quantization/compression \cite{zhao_towards_2023}. In this paper we will focus on the former two, which we review in the following.

As the name suggests, when applying partial participation only a subset of the agents is selected at each iteration to communicate their local models to the coordinator. This allows to reduce the number of agent-coordinator communications required \cite{nemeth_snapshot_2022}.
Local training proposes a complementary approach, by aiming to increase the time between communication rounds. The idea is for the agents to perform multiple epochs of local training (\textit{e.g.} via stochastic gradient descent) for each round of communication.
However, this may result in slower convergence due to the phenomenon of \textit{client drift} \cite{grudzien_can_2023}. This phenomenon also implies that the locally trained models drift apart, becoming biased towards the local data distribution, which is not (necessarily) representative of the overall distribution \cite[sec.~3.2]{karimireddy_scaffold_2020}. This further results in lower accuracy of the global trained model.
In this paper we will therefore focus on designing a federated learning algorithm that allows for both partial participation and local training -- without compromising accuracy.

\smallskip

Summarizing the discussion above, 
the objective of this paper is two-fold. On the one hand, we are interested in designing a privacy-preserving federated learning algorithm. On the other hand, we require that such algorithm employs partial participation and local training to reduce the number of communications, but without affecting accuracy.
The approach we propose to pursue these objectives is to leverage local training to enhance privacy, highlighting the synergy of communication reduction and privacy.
%
%
In particular, we offer the following contributions:
\begin{enumerate}
    \item We design a federated learning algorithm, abbreviated as \alg, based on the Peaceman-Rachford splitting method \cite{bauschke_convex_2017}, which allows for both partial participation and local training.
    In particular, the algorithm allows for a subset of the agents to activate randomly at each iteration, and communicate the results of their local training to the coordinator. The number of local training epochs is a tunable parameter of the algorithm.
    Moreover, the algorithm can be applied to both smooth and composite problems (with a smooth and non-smooth part). Non-smooth regularization terms can be used to encode prior knowledge in the structure of trained models \cite{iutzeler_nonsmoothness_2020}. Thus designing a federated algorithm that can efficiently solve this class of problems is important.

    \item We prove exact convergence (in mean) of \alg in strongly convex, composite scenarios when the agents use gradient descent during local training. In particular, we show there always exists a choice of parameters which ensures convergence.
    This result shows that the use of local training does not give rise to the client drift phenomenon, and thus it does not degrade accuracy of the trained model.
    Additionally, we analyze \alg's convergence when the agents employ accelerated gradient descent for local training, showing the flexibility of the proposed approach.
    
    \item We further analyze the convergence of \alg when stochastic gradient descent or noisy gradient descent are used for local training. The result characterizes convergence to a neighborhood of the optimal solution which depends on the variance of the noise.
    We then highlight the privacy benefits of using local training with a noisy solver. In particular, we characterize the differential privacy of the disclosed final model as a function of the number of local training epochs and the gradient noise. Our result shows that the privacy leakage does not exceed a finite bound for any choice of the algorithm's parameters, resulting in a tight privacy analysis.

    \item We apply \alg for a logistic regression problem, and compare its performance with a number of state-of-the-art federated algorithm that employ local training.
    We observe that the proposed algorithm compares quite favorably with the state of the art, outperforming it in some scenarios. We further apply the algorithms to a nonconvex problem, for which \alg also has good performance, pointing to interesting future research directions.
    We then analyze \alg by discussing its performance for different values of the tunable parameters and number of agents. One interesting observation -- aligning with \cite{mishchenko_proxskip_2022} -- is that the convergence rate is not a monotonically decreasing function of the number of local epochs (as may be suggested by \alg's foundation in the Peaceman-Rachford splitting). Rather, there exists an optimal, finite number of epochs, at least in this specific numerical set-up.
\end{enumerate}

\subsection{Related works}\label{subsec:theoretical-comparison}
(Private) federated learning is an active and growing area of research, and providing a comprehensive overview of the literature is outside the scope of this paper. In the following, we provide a brief overview of privacy in federated learning, and discuss the state of the art in federated algorithms employing local training.

Different privacy techniques have been proposed, with the most studied being homomorphic encryption \cite{zhang2020batchcrypt}, secret-sharing protocol \cite{zhang2021network}, and differential privacy (DP) \cite{dwork2014algorithmic}.
The latter stands out as easier to implement as compared to the other two methodologies.
This is indeed the case as in differential privacy the aim is to obscure local training data by introducing perturbations to the models shared by the agents. This is achieved by simply randomizing learning algorithms with additive noise (often following a normal or Laplace distribution). But despite the simplicity of such privacy mechanism, DP  provides a solid framework that enables the derivation of theoretical privacy guarantees \cite{chourasia2021differential,chaudhuri2011differentially}.
We remark that recently, the use of random additive noise to improve privacy has also been complemented with other techniques such as quantization and compression of communications \cite{amiri2021compressive,pmlr-v180-chaudhuri22a,lang_joint_2023}. This is similar to the concept of employing local training to enhance privacy, as explored in this paper.
Owing to its effectiveness, differential privacy has been fruitfully applied in diverse engineering fields, ranging from federated learning \cite{gafni_federated_2022,wei_federated_2020,noble_differentially_2022} and distributed optimization \cite{han_differentially_2017}, to control system design \cite{han_privacy_2018}, estimation \cite{le_ny_differentially_2014}, and consensus-seeking \cite{wang2024differentially}. Thus we will follow these works in using DP throughout this paper.

We turn now to reviewing related works on federated learning with local training.
An early example is that of FedSplit, proposed in \cite{pathak_fedsplit_2020}; this algorithm shares the same foundation -- the Peaceman-Rachford splitting -- as the proposed \alg. However, FedSplit does not adopt the local training initialization which allows to prove exact convergence as we do in this paper. Moreover, it does not explicitly allow for different local solvers (including privacy-preserving ones), for partial participation, or for composite cost function. Thus, the algorithmic design and results in the following extend beyond the scope of \cite{pathak_fedsplit_2020}.
More recent algorithms which explicitly incorporate local training are instead FedPD \cite{zhang_fedpd_2021}, FedLin \cite{mitra_linear_2021}, TAMUNA \cite{condat_tamuna_2023}, LED \cite{alghunaim_local_2023}, 5GCS \cite{grudzien_can_2023}.
Their comparison with the proposed \alg is summarized in \cref{tab:comparison-algorithms}.

Let us start by evaluating the memory footprint of the different algorithms. In \cref{tab:comparison-algorithms} we report the number of models that need to be stored by agents and coordinator in between communication rounds. As we can see, except for \cite{grudzien_can_2023}, all other methods -- including \alg -- require the storage of $2 N + 1$ variables, \textit{i.e.} two models per agent and one at the coordinator.

Turning now to the local training, we can classify algorithms based on the solvers that can be applied by the agents. For \cite{zhang_fedpd_2021,mitra_linear_2021,condat_tamuna_2023,alghunaim_local_2023} only (stochastic) gradient descent is allowed, while \alg also allows for accelerated (S)GD and noisy GD (to enhance privacy). Only \cite{grudzien_can_2023} potentially allows for a broader class of solvers, provided that they verify a specific descent inequality. Such inequality is verified by SGD and variance reduced SGD, but \cite{grudzien_can_2023} does not evaluate if accelerated SGD verifies it.
The compared algorithms also differ in the number of local training epochs that the agents can perform. On one side, we have \cite{condat_tamuna_2023} which implements local training by randomly choosing whether to perform a communication round at each gradient step. On the other we have \cite{zhang_fedpd_2021,mitra_linear_2021,alghunaim_local_2023,grudzien_can_2023} and \alg, which apply a deterministic number of local epochs $\Ne$. However, for \cite{zhang_fedpd_2021,grudzien_can_2023} convergence is guaranteed only if $\Ne$ is lower bounded by $\ubar{N}_\mathrm{e}$, which may be larger than $1$.

Partial participation is allowed only in the proposed \alg and \cite{condat_tamuna_2023,grudzien_can_2023}, and not in \cite{zhang_fedpd_2021,mitra_linear_2021,alghunaim_local_2023}.
Furthermore, only \alg includes a privacy preserving mechanism with a rigorous differential privacy guarantee (derived in \cref{sec:privacy}). All the other methods do not evaluate the privacy perspective.

Finally, the algorithms also differ in the classes of problems for which a convergence guarantee can be derived. \cite{zhang_fedpd_2021,mitra_linear_2021,alghunaim_local_2023} present more general results for non-convex problems, while the assumption of strong convexity is required in the analysis of \alg and \cite{condat_tamuna_2023,grudzien_can_2023}.
However, \alg can be applied on composite (convex) problems, differently from all the other algorithms that are designed for smooth (but possibly non-convex) problems only.

In summary, this paper presents an algorithmic design platform for federated learning which allows for communication efficiency (through partial participation and local training), customizable local computations, privacy preserving mechanisms, and which can be applied to composite costs (although our analysis currently supports only strongly convex problems). As discussed above, many of these features have been presented separately in the literature; however, our proposed solution merges them in a unified and modular framework.

\begin{table*}[!ht]
\centering
\caption{Comparison of different federated algorithms with local training.}
\label{tab:comparison-algorithms}
    \begin{tabular}{cccccccc}
    \hline
    Algorithm [Ref.] & \thead{\# variables \\ stored} & local solvers & \# local epochs & \thead{partial \\ participation} & privacy & \thead{problem \\ assumptions} \\
    \hline
    \textbf{\alg} [this work] & $2 N + 1$& \thead{(S)GD, Acc. (S)GD \\ noisy GD} & $\geq 1$ & \cmark & \cmark & str. convex, composite \\
    FedPD \cite{zhang_fedpd_2021} & $2 N + 1$ & (S)GD & $\geq \ubar{N}_\mathrm{e}$ & \xmark & \xmark & non-convex \\
    FedLin \cite{mitra_linear_2021} & $2 N + 1$ & (S)GD & $\geq 1$ & \xmark & \xmark & non-convex \\
    TAMUNA \cite{condat_tamuna_2023} & $2 N + 1$ & (S)GD & random$^\dagger$ & \cmark & \xmark & str. convex \\
    LED \cite{alghunaim_local_2023} & $2 N + 1$ & (S)GD & $\geq 1$ & \xmark & \xmark & non-convex \\
    5GCS \cite{grudzien_can_2023} & $N + 3$ & any$^*$ & $\geq \ubar{N}_\mathrm{e}$ & \cmark & \xmark & str. convex \\
    \hline
    \end{tabular}
    \\\vspace{0.1cm}
    $^\dagger$ The number of local epochs at time $k$ is drawn from $\text{Geom}(p)$. \quad
    $^*$ Provided that the solver verifies a specific descent inequality.
\end{table*}

\smallskip

\paragraph*{Notation}
In the paper we denote the class of \textit{convex, closed and proper} functions as $\ccp{\R^n}$. The class of \textit{$\lmin$-strongly convex and $\lmax$-smooth} functions is denoted by $\scs{\lmin}{\lmax}{\R^n}$.
Bold letters represent vectors \textit{e.g.} $\x \in \R^n$, and calligraphic capital letters represent operators \textit{e.g.} $\T : \R^n \to \R^n$.
$\1$ represents the vector or matrix of all ones.
Let $\C \subset \R^n$, then we denote the indicator function of the set by $\iota_\C$, with $\iota_\C(\x) = 0$ if $\x \in \C$, $+\infty$ otherwise.

\section{Preliminaries}\label{sec:preliminaries}

\subsection{Operator theory}
In the following we review some notions in \textit{operator theory}
\footnote{Following the convention in the convex optimization literature, we use the term \textit{operator}, even though in finite-dimensional spaces \textit{mapping} would be more appropriate.};
for a comprehensive background we refer to \cite{bauschke_convex_2017,ryu_primer_2016}.
The central goal in operator theory is to find the fixed points of a given operator.

\begin{definition}[Fixed points]
Consider the operator $\T : \R^n \to \R^n$, we say that $\bar{\x} \in \R^n$ is a \emph{fixed point} of $\T$ if $\bar{\x} = \T \bar{\x}$.
\end{definition}

The properties of an operator characterize its fixed points and dictate what techniques can be used to find them. An especially amenable class is that of contractive operators, defined below, for which the Banach-Picard theorem \cite[Theorem~1.50]{bauschke_convex_2017}, reported in \cref{lem:banach-picard}, holds.

\begin{definition}[Contractive operators]\label{def:contractive-operator}
The operator $\T$ is \emph{$\zeta$-contractive}, $\zeta \in (0, 1)$, if:
\begin{equation}
	\norm{\T \x - \T \y} \leq \zeta \norm{\x - \y}, \quad \forall \x, \y \in \R^n.
\end{equation}
\end{definition}

\begin{lemma}[Banach-Picard]\label{lem:banach-picard}
Let $\T : \R^n \to \R^n$ be $\zeta$-contractive; then $\T$ has a unique fixed point $\bar{\x}$, which is the limit of the sequence generated by:
\begin{equation}\label{eq:banach-picard}
    \x^{\ell+1} = \T \x^\ell, \quad \ell \in \N, \ \x^0 \in \R^n.
\end{equation}
In particular, it holds
$
    \norm{\x^\ell - \bar{\x}} \leq \zeta^\ell \norm{\x^0 - \bar{\x}}.
$
\end{lemma}

\subsection{Application to convex optimization}\label{subsec:optimization-operator}
Operator theory can be applied to solve convex optimization problems by reformulating them into fixed point problems. In the following we review two operator-based solvers that will be useful throughout the paper.

\begin{lemma}[Gradient descent]\label{lem:gradient-descent}
Let $f \in \scs{\lmin}{\lmax}{\R^n}$, then the gradient descent operator
$$
    \T = \I - \rho \nabla f(\cdot)
$$
with $\rho \in (0, 2 / \lmax)$ is $\zeta$-contractive, $\zeta = \max\{ |1 - \rho \lmin|, |1 - \rho \lmax| \}$ \cite[p.~15]{ryu_primer_2016}, and the unique fixed point of $\T$ coincides with the minimizer of $f$.
\end{lemma}

Clearly, the gradient descent algorithm can be applied to solve the smooth problem $\min_\x f(\x)$. To solve the \textit{composite} problem $\min_\x f(\x) + g(\x)$, with $f$ smooth and $g$ non-smooth, we need to resort to different algorithms, such as the \textit{Peaceman-Rachford splitting} (PRS) of \cref{lem:prs}. To introduce the PRS we first need to define the proximal and reflective operators.

\begin{definition}[Proximal, reflective operators]\label{def:proximal}
Let $f \in \ccp{\R^n}$ and $\rho > 0$, then the \textit{proximal} operator of $f$ at $\y \in \R^n$ with penalty $\rho$ is defined as
$$
    \prox_{\rho f}(\y) = \argmin_{\x \in \R^n} \left\{ f(\x) + \frac{1}{2 \rho} \norm{\x - \y}^2 \right\}.
$$
The \textit{reflective} operator of $f$ at $\y \in \R^n$ is then defined as
$
    \refl_{\rho f}(\y) = 2 \prox_{\rho f}(\y) - \y.
$
\end{definition}

Notice that by definition, the proximal of a convex function $f$ is unique, as it is the minimizer of a strongly convex problem. Moreover, if $f$ is the indicator function of a convex set $\C$, then its proximal coincides with the \textit{projection} onto the set, denoted by $\proj_\C$.
In the following we call a function $f$ \textit{proximable} if its proximal (and hence reflective) can be computed in closed form or in a computationally inexpensive way, see \cite{parikh_proximal_2014} for some examples.

We are now ready to introduce the PRS.

\begin{lemma}[Peaceman-Rachford splitting]\label{lem:prs}
Consider the optimization problem $\bar{\x} = \argmin_\x f(\x) + g(\x)$ with $f \in \scs{\lmin}{\lmax}{\R^n}$ and $g \in \ccp{\R^n}$.
The corresponding Peaceman-Rachford splitting is characterized by the operator \cite[Proposition~28.8]{bauschke_convex_2017}
$$
    \T = \refl_{\rho f} \circ \refl_{\rho g}
$$
for $\rho > 0$.
The operator $\T$ is $\zeta$-contractive \cite[Theorem~2]{giselsson_linear_2017}, with
$$
    \zeta = \max\left\{ \left\vert \frac{1 - \rho \lmax}{1 + \rho \lmax} \right\vert, \left\vert \frac{1 - \rho \lmin}{1 + \rho \lmin} \right\vert \right\},
$$
and its unique fixed point $\bar{\z}$ is such that $\bar{\x} = \prox_{\rho g}(\bar{\z}) = \prox_{\rho f}(2 \bar{\x} - \bar{\z})$.
\end{lemma}

\smallskip

Applying the Banach-Picard iteration~\cref{eq:banach-picard} to the PRS operator defined in \cref{lem:prs} yields the updates \cite[section~3]{giselsson_linear_2017}
\begin{subequations}\label{eq:prs}
\begin{align}
    \y_{k+1} &= \prox_{\rho g}(\z_k) \\
    \x_{k+1} &= \prox_{\rho f}(2 \y_{k+1} - \z_k) \\
    \z_{k+1} &= \z_k + 2 (\x_{k+1} - \y_{k+1})
\end{align}
\end{subequations}
which will serve as the template for the algorithm designed in the \cref{sec:algorithm}.

\subsection{Stochastic operator theory}\label{subsec:stochastic-operators}
In many learning applications the methods described in the previous \cref{subsec:optimization-operator} cannot be directly applied due to practical constraints. For example, the gradient descent of \cref{lem:gradient-descent} may only have access to stochastic gradients \cite{gorbunov_unified_2020}. And, in decentralized set-ups, asynchronous computations may result in only some of the components of $\x$ being updated at any given time.
In these scenarios, optimization algorithms are modeled by \textit{stochastic operators}, which we briefly review in this section, see \cite{bastianello_stochastic_2022} and references therein.

Consider the operator $\T : \R^{n N} \to \R^{n N}$ that maps $\x = [x_1^\top, \ldots, x_N^\top]^\top$, $x_i \in \R^n$, into $\T \x = [(\T_1 \x)\top, \ldots, (\T_N \x)^\top]^\top$.
We define the \textit{stochastic Banach-Picard} for $\T$ as the update
\begin{equation}\label{eq:stochastic-bp}
    x_{i,k+1} = \begin{cases}
        \T_i \x_k + e_{i,k} & \text{if} \ u_{i,k} = 1 \\
        x_{i,k} & \text{otherwise}
    \end{cases}
\end{equation}
where $u_{i,k} \sim Ber(p_i)$, $p_i \in (0, 1]$, models randomized coordinate updates, and the random vector $e_{i,k}$ models additive errors, for example stemming from the use of stochastic gradients. We will denote $\e_k = [e_{1,k}^\top, \ldots, e_{N,k}^\top]^\top$.
The following result can be derived as a stochastic version of \cref{lem:banach-picard}, see \cite[Proposition~1]{bastianello_stochastic_2022}.

\begin{lemma}[Stochastic Banach-Picard]\label{lem:stochastic-banach-picard}
Assume that $\T$ in~\cref{eq:stochastic-bp} is $\zeta$-contractive with fixed point $\bar{\x}$. Assume that $\{ u_{i,k} \}_{k \in \N}$ and $\{ e_{i,k} \}_{k \in \N}$ are i.i.d. and independent of each other, and $\ev{\norm{\e_k}} = \nu < \infty$.
Then the random sequence $\{ \x_k \}_{k \in \N}$ generated by~\cref{eq:stochastic-bp} with initial condition $\x_0$ satisfies the following bound
$$
    \ev{\norm{\x_k - \bar{\x}}} \leq \sqrt{\frac{\pmax}{\pmin}} \left( \bar{\zeta}^k \norm{\x_0 - \bar{\x}} + \frac{1 - \bar{\zeta}^k}{1 - \bar{\zeta}} \nu \right)
$$
where $\pmin = \min_i p_i$, $\pmax = \max_i p_i$, and $\bar{\zeta} = \sqrt{1 - \pmin + \pmin \zeta^2}$.
\end{lemma}

\subsection{Differential privacy}


In this section we review some concepts in differential privacy.
Let $\mathcal{D}=\{\xi^{1},\dots, \xi^{q}  \}$ denote a dataset of size $q$ with records drawn from a universe $\mathbb{X}$. Two datasets $\mathcal{D}$ and $\mathcal{D}'$ are referred to as \textit{neighboring} if they are of the same size and differ in at most one data point.

\begin{definition}[R\'enyi differential privacy]
Given $\lambda >1$, a randomized mechanism $\mathcal{M}: \mathbb{X}^q \rightarrow \mathbb{Y}$ is said to have $(\lambda,\varepsilon)$-R\'enyi differential privacy (RDP) \cite{mironov2017renyi}, if for every pair of neighboring datasets $\mathcal{D}$, $\mathcal{D}' \in \mathbb{X}^q$, we have	\begin{equation*}
    {\rm D}_{\lambda}(\mathcal{M}(\mathcal{D})||\mathcal{M}(\mathcal{D}')) \leq \varepsilon,
\end{equation*}
where ${\rm D}_{\lambda}(\mathcal{M}(\mathcal{D})||\mathcal{M}(\mathcal{D}'))$ is the $\lambda$-R\'enyi divergence between $\mathcal{M}(\mathcal{D})$ and $\mathcal{M}(\mathcal{D}')$, i.e.,
\begin{equation*}
    \begin{split}
        &{\rm D}_{\lambda}(\mathcal{M}(\mathcal{D})||\mathcal{M}(\mathcal{D}')) \\
        &= \frac{1}{\lambda -1} \log \int_{\mathbb{Y}} \mathbb{P}[\mathcal{M}(\mathcal{D})=z]^{\lambda}\mathbb{P}[\mathcal{M}(\mathcal{D}')=z]^{1-\lambda} dz.
    \end{split}
\end{equation*}
\end{definition}

Another common notion in differential privacy is approximate differential privacy (ADP) \cite{dwork2014algorithmic}. In the following we introduce its definition and how RDP can be translated into ADP \cite{mironov2017renyi}.

\begin{definition}[Approximate differential privacy]
Given $\varepsilon,\delta\geq 0$, a randomized mechanism $\mathcal{M}: \mathbb{X}^q \rightarrow \mathbb{Y}$ is said to be $(\varepsilon,\delta)$-ADP, if for every pair of neighboring datasets $\mathcal{D}$, $\mathcal{D}'\in\mathbb{X}^q$ and every subset $\mathcal{O}\subseteq \mathbb{Y}$, we have
\begin{equation*}
\begin{split}
    \mathbb{P}[\mathcal{M}(\mathcal{D})\in\mathcal{O}]\leq  e^{\varepsilon} \mathbb{P}[\mathcal{M}(\mathcal{D'})\in\mathcal{O}]+\delta.
\end{split}
\end{equation*}
\end{definition}

\begin{lemma}[RDP to ADP conversion]
\label{lem:DP_Conv}
	If a randomized algorithm is ($\lambda,\varepsilon$)-RDP, then it is $(\varepsilon+\frac{\log(1/\delta)}{\lambda-1}, \delta)$-ADP, $\forall \delta \in(0,1)$.
\end{lemma}

\section{Problem Formulation}\label{sec:problem-formulation}
In the following we formally describe the problem at hand, and discuss the application-driven design objectives that the proposed algorithm should satisfy.

\subsection{Problem}

\begin{figure}[!ht]
    \centering
    \includegraphics{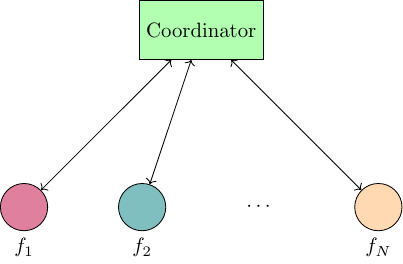}
    \caption{The federated architecture.}
    \label{fig:federated-architecture}
\end{figure}

Consider the federated learning set-up stylized in \cref{fig:federated-architecture}, in which $N$ agents and a coordinator cooperate towards the solution of the \textit{composite, empirical risk} minimization problem
\begin{equation}\label{eq:problem}
    \bar{x} = \argmin_{x \in \R^{n}} \ \sum_{i = 1}^N f_i(x) + h(x)
\end{equation}
where
$
    f_i(x) := \frac{1}{q_i}\sum_{j=1}^{q_i} \ell_i\left( x;\xi_i^{(j)} \right)
$
is the local cost function of agent $i$ characterized by a local dataset $\mathcal{D}_i = \{ \xi_{i}^{(1)}, \dots, \xi_{i}^{(q_i)} \}$, $\ell_i(x,\xi)$ is the loss of the machine learning model over the data instance $\xi$, and $h : \R^n \to \R \cup \{ +\infty \}$ is a common regularization function, which may be non-smooth. Non-smooth regularization terms are particularly important in machine learning, as they can be used to encode prior knowledge in the structure of trained models \cite{iutzeler_nonsmoothness_2020}. For example, the regularizer $h(x) = \norm{x}_1$ is often used to promote sparsity (low complexity) of the trained models. We also remark that our algorithm directly applies to smooth problems when setting $h(x) = 0$.
We introduce the following assumption.

\begin{assumption}\label{as:costs}
The local loss functions $\ell_i$, $i \in \{ 1, \ldots, N \}$, are $\lmin$-strongly convex and $\lmax$-smooth with respect to the first argument. The cost $h$ is convex, closed, and proper\footnote{A function $h : \R^n \to \R$ is \textit{closed} if for any $t \in \R$ the set $\{ \x \in \R^n \ | \ f(\x) \leq t \}$ is closed; it is \textit{proper} if $h(x) > -\infty$ for any $x \in \R^n$.}.
\end{assumption}

We remark that \cref{as:costs} implies that each $f_i$ is both $\lmin$-strongly convex and $\lmax$-smooth. A standard assumption in composite optimization is for $h$ to be proximable, that is, to have a closed-form proximal \cite{parikh_proximal_2014,iutzeler_nonsmoothness_2020}. Many costs of interest,  (such as the sparsity promoting $h(x) = \norm{x}_1$) are indeed proximable. However, this is not required in our set-up; indeed, if $h$ is not proximable we can approximate $\prox_{\rho h}$ to some predefined accuracy, to the cost of injecting an additive error, which is nonetheless allowed by \cref{pr:convergence}.

In order to solve the problem in a federated set-up, the first step is to reformulate it into the following
\begin{equation}\label{eq:distributed-problem}
\begin{split}
    \bar{\x} =& \argmin_{\x \in \R^{n N}} \ \sum_{i = 1}^N f_i(x_i) + h(x_1) \quad \text{s.t.} \ x_1 = \ldots = x_N
\end{split}
\end{equation}
where $x_i \in \R^n$ are local copies of $x$, and $\x = [x_1^\top, \ldots, x_N^\top]$.
The consensus constraint $x_1 = \ldots = x_N$ guarantees the equivalence of~\cref{eq:distributed-problem} with~\cref{eq:problem}, and implies that $\bar{\x} = \1_N \otimes \bar{x}$, where $\bar{x}$ is the unique optimal solution of~\cref{eq:problem}.
We can now rewrite~\cref{eq:distributed-problem} in the following form
\begin{equation}\label{eq:distributed-problem-2}
    \bar{\x} = \argmin_{\x \in \R^{n N}} f(\x) + g(\x)
\end{equation}
with $f(\x) = \sum_{i = 1}^N f_i(x_i)$, and
$
    g(\x) = \iota_\C(\x) + h(x_1),
$
where we imposed the consensus constraints through the indicator function of the \textit{consensus subspace}
$$
    \C = \{ \x = [x_1^\top, \ldots, x_N^\top] \in \R^{n N} \ \vert \ x_1 = \ldots = x_N \}.
$$

\smallskip

In principle,~\cref{eq:distributed-problem-2} could be now solved applying the Peaceman-Rachford splitting of~\cref{eq:prs}, but due to practical constraints (discussed in \cref{subsec:design-objectives}) this may not be possible. Instead, in \cref{sec:algorithm} we will show how to use PRS \textit{as a template} to design a federated algorithm that can be applied in practice.

\subsection{Design objectives}\label{subsec:design-objectives}
In this section we discuss some of the challenges that arise in federated learning set-ups, due to the practical constraints that they pose \cite{li_federated_2020,gafni_federated_2022}. The objective in the remainder of the paper then is to design and analyze an algorithm that can provably overcome these challenges.

\paragraph{Communication burden}
In machine learning we are often faced with the task of training high-dimensional models, with deep neural networks the foremost example. This implies that problem~\cref{eq:problem} may be characterized by $n \gg 1$.
In a federated set-up then, sharing models between the agents and the coordinator imposes a significant communication burden \cite{zhao_towards_2023}.
Initial approaches to federated learning\footnote{Corresponding to the $1^\text{st}$--$3^\text{rd}$ generations in the categorization of \cite{grudzien_can_2023}.} therefore introduced different techniques to reduce the number of communications required, such as \textit{local training} \cite{gafni_federated_2022}.
Local training requires that, before a round of communications, each agent perform multiple steps (or epochs) of minimization of the local loss function.
Employing local training, however, may result in worse accuracy, especially when the loss functions of the agents are heterogeneous. The objective then is to design federated algorithms that perform local training without jeopardizing accuracy of the trained model, see \textit{e.g.} \cite{zhang_fedpd_2021,mitra_linear_2021,condat_tamuna_2023,grudzien_can_2023,alghunaim_local_2023}.

\paragraph{Heterogeneous agents}
The devices involved in the federated learning process may be in general very heterogeneous as regards to computational, storage, and communication capabilities. Indeed, these devices may be equipped with very different hardware (CPU, memory, battery), and they may be using different network connectivity (mobile or Wi-Fi) \cite{li_federated_2020,gafni_federated_2022}.
The heterogeneity of the agents' resources implies that they may perform local computations, and transmit their results, at different rates. This results in \textit{partial participation}, with only a subset of the agents being active at any given time.
Besides arising from practical limitations, partial participation may also be enforced in order to relieve the communication burden.
The goal therefore is to design a federated algorithm that allows for partial participation while preserving accuracy of the trained model.

\paragraph{Privacy}
The leakage of private information through machine learning models' parameters and predictions has been extensively documented \cite{nasr2019comprehensive,bassily2014private,shokri2017membership}. One notable example is the support vector machine (SVM) in its dual form, where the solution involves a linear combination of specific data points. Recently, there has been a growing interest in iterative randomized training algorithms that offer privacy-preserving guarantees \cite{chourasia2021differential,chaudhuri2011differentially}, such as RDP and ADP. However, the existing results suffer from privacy loss that increases over training epochs \cite{liu2023privacy} and fail to support local training and the partial participation of heterogeneous agents \cite{chourasia2021differential}. These limitations make it challenging to extend those approaches to the specific setup addressed in this work. The objective of this study is to address these challenges by providing a rigorous privacy-preserving guarantee that remains bounded over time and is compatible with the federated learning setup under consideration.

\smallskip

To summarize, the goal of this paper is to design a federated algorithm with the following features:
\begin{enumerate}
    \item[(i)] \textit{Local training}: the algorithm should allow for multiple epochs of local training, but without jeopardizing accuracy of the computed solution.

    \item[(ii)] \textit{Partial participation}: the algorithm should allow for a subset of the agents to be active at any time, again without jeopardizing accuracy.

    \item[(iii)] \textit{Privacy}: the algorithm should incorporate a privacy mechanism to safeguard the local training data.
\end{enumerate}

\section{Algorithm Design}\label{sec:algorithm}
In this section we derive and discuss the proposed \cref{alg:main-algorithm}, highlighting the different design choices that it allows for.

\begin{algorithm}[!ht]
\caption{Federated Private Local Training (\alg)}
\label{alg:main-algorithm}
\begin{algorithmic}[1]
	\Require For each agent initialize $x_{i,0}$ and $z_{i,0}$; choose the local solver, its parameters, and the number of local epochs $\Ne$; the parameter $\rho > 0$ 
   \Ensure $x_{i,K}$
	\For{$k = 0, 1, \ldots,K$ each agent $i$}
        \CommentState{coordinator}
        \State the coordinator averages the information received from the agents
        $$
            y_{k+1} = \prox_{\rho h / N}\left(\frac{1}{N} \sum_{i = 1}^N z_{i,k} \right)
        $$
        and transmits it to all agents

    	\CommentState{agents}
        \For{$i = 1, \ldots, N$}
            \If{agent $i$ is active}
                \State set $w_{i,k}^0 = x_{i,k}$, $v_{i,k} = 2 y_{k+1} - z_{i,k}$
                \For{$\ell = 0, 1, \ldots, \Ne-1$} \Comment{local solver}
                    \State \textit{e.g.} gradient descent
                    $$
                        \hspace{-0.6cm}w_{i,k}^{\ell+1} = w_{i,k}^\ell - \gamma \left( \nabla f_i(w_{i,k}^\ell) + \frac{1}{\rho} (w_{i,k}^\ell - v_{i,k})  \right)
                    $$
                \EndFor
                \State set $x_{i,k+1} = w_{i,k}^{\Ne}$
                \State update the auxiliary variable
                $$
                    z_{i,k+1} = z_{i,k} + 2 (x_{i,k+1} - y_{k+1})
                $$
                and transmit to coordinator
            \Else
                \Comment{inactive agent}
                \State set $x_{i,k+1} = x_{i,k}$ and $z_{i,k+1} = z_{i,k}$
            \EndIf
        \EndFor

        \State 
        
    \EndFor
\end{algorithmic}
\end{algorithm}

\subsection{Algorithm derivation}
We start by applying the PRS of~\cref{eq:prs} to the federated optimization problem~\cref{eq:distributed-problem-2}, which yields
\begin{subequations}\label{eq:prs-fed}
\begin{align}
    \y_{k+1} &= \prox_{\rho g}(\z_k) \label{eq:prs-fed-y} \\
    \x_{k+1} &= \prox_{\rho f}(2 \y_{k+1} - \z_k) \label{eq:prs-fed-x} \\
    \z_{k+1} &= \z_k + 2 (\x_{k+1} - \y_{k+1}), \label{eq:prs-fed-z}
\end{align}
\end{subequations}
We now discuss the implementation of \cref{eq:prs-fed-y} and \cref{eq:prs-fed-x} in turn.

\subsubsection{Implementation of~\cref{eq:prs-fed-y}}
First of all, update~\cref{eq:prs-fed-y} can be computed by the coordinator using the information it receives from all the agents, as detailed in \cref{lem:prox-g}.

\begin{lemma}[$\prox_{\rho g}$ computation]\label{lem:prox-g}
The proximal of $g(\x) = \iota_\C(\x) + h(x_1)$ at a point $\z \in \R^{n N}$ is given by
$
    \prox_{\rho g}(\z) = \1_N \otimes \prox_{\rho h / N}\left(\frac{1}{N} \sum_{i = 1}^N z_i \right).
$
\end{lemma}
\begin{proof}
See Appendix~\ref{proof:sec:algorithm}.
\end{proof}

\smallskip

Given the result of \cref{lem:prox-g}, the coordinator is therefore tasked with storing and updating the variable
$
    y_{k+1} = \prox_{\rho h / N}\left(\frac{1}{N} \sum_{i = 1}^N z_{i,k} \right)
$
for which $\y_{k+1} = \1_N \otimes y_{k+1}$.

\smallskip

\subsubsection{Implementation of~\cref{eq:prs-fed-x}}
Computing~\cref{eq:prs-fed-x} instead is a task for the agents, as the update depends on the private local costs. In particular, letting $\vv_k = 2 \y_{k+1} - \z_k$, we notice that the update~\cref{eq:prs-fed-x} requires the solution of
\begin{align}
    &\min_\x \sum_{i = 1}^N f_i(x_i) + \frac{1}{2 \rho} \norm{\x - \vv_k}^2 = \label{eq:online-problem} \\
    & \hspace{1cm} = \sum_{i = 1}^N \min_{x_i} f_i(x_i) + \frac{1}{2 \rho} \norm{x_i - v_{i,k}}^2 =: d_{i,k}(x_i), \nonumber
\end{align}
where $v_{i,k} = 2 y_{k+1} - z_{i,k}$.
Since the problem can be decoupled, each agent needs to locally solve the corresponding minimization $\min_{x_i} d_{i,k}(x_i)$, which yields the local state update\footnote{This can also be seen by directly applying \cite[section~2.1]{parikh_proximal_2014}.}
\begin{equation}\label{eq:local-update}
    x_{i,k+1} = \prox_{\rho f_i}(v_{i,k}) = \argmin_{x_i \in \R^n} d_{i,k}(x_i).
\end{equation}
However, in general~\cref{eq:local-update} does not have a closed form solution, and the agents can only \textit{approximate the update via a finite number of local training epochs}.
\footnote{Even if a closed form solution is in principle available, we may not be able to actually compute it due to the high dimensionality of the problem. For example, if $f_i(x) = \frac{1}{2} x^\top A_i x + \langle b_i, x \rangle$ then $\prox_{\rho f_i}(v_{i,k}) = \left( I_n + \rho A_i \right)^{-1} (v_{i,k} - b_i)$, which requires inversion of an $n \times n$ matrix. Therefore if $n \gg 1$, the matrix inversion may be computationally expensive and thus the closed form update may be unavailable in practice.}
As a first approach, each agent can apply $\Ne \in \N$ steps of gradient descent (cf. \cref{lem:gradient-descent}) to approximate the local update:
\begin{equation}\label{eq:local-update-approx}
\begin{split}
    w_{i,k}^0 &= x_{i,k} \\
    w_{i,k}^{\ell+1} &= w_{i,k}^\ell - \gamma \nabla d_{i,k}(w_{i,k}^\ell) \quad \ell = 0, \ldots, \Ne-1 \\
    x_{i,k+1} &= w_{i,k}^\Ne
\end{split}
\end{equation}
where $\gamma \in (0, 2 / (\lmin + 1 / \rho))$, since $d_{i,k} \in \scs{\lmin + 1/\rho}{\lmax + 1/\rho}{\R^n}$. More generally, one can \textit{apply any suitable local solver} to compute~\cref{eq:local-update}, as discussed in \cref{subsec:local-training}. In particular, the choice of local solver can be leveraged to \textit{enhance privacy}.

Finally, in the derivation above, each agent approximates its local proximal at each iteration $k \in \N$. However, this is not required, as \cref{alg:main-algorithm} allows for only a subset of \textit{active} agents to share the results of their local training at iteration $k \in \N$; this feature is discussed in \cref{subsec:partial-participation}.

\vspace{0.2cm}

To summarize, the proposed \cref{alg:main-algorithm} provides two flexible design choices in the following key aspects:
\begin{enumerate}
    \item \textit{local training}: during local training, any suitable solver can be employed by the agents, including privacy preserving methods (objectives (i) and (iii));
    \item \textit{partial participation}: the algorithm allows for a \textit{subset} of the agents to be active at any given time $k \in \N$ (objective (ii)).
\end{enumerate}

\subsection{Local training methods}\label{subsec:local-training}
In principle, one may choose to use any suitable local solver to approximate the solution of $\min_{x_i} d_{i,k}(x_i) = f_i(x_i) + (1 / 2\rho) \norm{x_i - v_{i,k}}^2$. This problem is especially amenable since it is strongly convex, with $d_{i,k} \in \scs{\lmin + 1/\rho}{\lmax + 1/\rho}{\R^n}$.
A first option is the gradient descent, already discussed in~\cref{eq:local-update-approx}. Notice that by \cref{lem:gradient-descent} we know that the gradient descent is contractive with rate $\chi = \max\{ |1 - \gamma (\lmin+1/\rho)|, |1 - \gamma (\lmax+1/\rho)| \}$, which is minimized by choosing $\gamma = 2 / (\lmax + \lmin + 2/\rho)$. In the following, we discuss some alternatives that address different design objectives.

\subsubsection*{Accelerated gradient descent}
Improving the performance of local training can be achieved by using the \textit{accelerated gradient descent}, see \textit{e.g.} \cite[Algorithm~14]{daspremont_acceleration_2021}, characterized by
\begin{equation}\label{eq:local-update-accelerated}
\begin{split}
    w_{i,k}^0 &= x_{i,k}, \ u_i^0 = w_{i,k}^0 \\
    u_i^{\ell+1} &= w_{i,k}^\ell - \frac{1}{\lmax + 1/\rho} \nabla d_{i,k}(w_{i,k}^\ell) \quad \ell = 0, \ldots, \Ne-1 \\
    w_{i,k}^{\ell+1} &= u_i^{\ell+1} + \frac{\sqrt{\lmax + 1/\rho} - \sqrt{\lmin + 1/\rho}}{\sqrt{\lmax + 1/\rho} + \sqrt{\lmin + 1/\rho}} (u_i^{\ell+1} - u_i^\ell) \\
    x_{i,k+1} &= w_{i,k}^\Ne,
\end{split}
\end{equation}
which employs constant step-sizes owing to the strong convexity of $d_{i,k}$.
Local solvers with even better performance can also be chosen, for example (quasi-)Newton methods.

\subsubsection*{Stochastic gradient descent}
Both gradient descent and accelerated gradient descent rely on the use of full gradient evaluations, which may be too computationally expensive in learning applications. Therefore, the agents may resort to \textit{stochastic gradient descent} (SGD) methods \cite{gorbunov_unified_2020}, which employ approximated gradients during the local training.
Recalling that the local costs are defined as
$
    f_i(x) := \frac{1}{q_i} \sum_{j=1}^{q_i} \ell_i(x;\xi_i^{(j)}),
$
SGD makes use of the following approximate gradient
$$
    \hat{\nabla} f_i(x) = \frac{1}{|\mathcal{B}_i|} \sum_{j \in \mathcal{B}_i} \nabla \ell_i(x;\xi_i^{(j)})
$$
where $\mathcal{B}_i \subset \{ 1, \ldots, q_i \}$ are the indices of a subset of data points chosen uniformly at random.

\subsubsection*{Private local training}
The previous local solvers address performance concerns; however, when the design objective is privacy preservation, the agents should resort to the use of \textit{noisy gradient descent} \cite{chourasia2021differential,altschuler_privacy_2022}. In particular, the update of $w_{i,k}^{\ell+1}$ in~\cref{eq:local-update-approx} is perturbed with Gaussian noise:
\begin{equation}\label{eq:private-local-update-approx}
\begin{split}
    w_{i,k}^{\ell+1} &= w_{i,k}^\ell - \gamma \nabla d_{i,k}(w_{i,k}^\ell) + t_{i,k}^\ell \\
    t_{i,k}^\ell &\sim \sqrt{2\gamma} \mathcal{N}(0,\tau^2{I}_n)
\end{split}
\end{equation}
where $ \ell = 0, \ldots, \Ne-1 $, and $\tau^2 > 0$ is the noise variance.
The addition of a stochastic additive noise in each iteration of~\cref{eq:private-local-update-approx} allows us to establish rigorous differential privacy guarantees, which can be tuned by selecting the variance $\tau^2$.

\smallskip

\begin{remark}[Uncoordinated local solvers]
In this section we only discussed the case of ``coordinated'' local solvers, that is, all agents use the same solver with the same parameters (\textit{e.g.} gradient with common -- or coordinated -- step-size $\gamma$).
The results of \cref{sec:convergence} actually allow for uncoordinated solvers, with the agents either using different algorithms, or the same algorithm with different parameters. For example, if the local costs are such that $f_i \in \scs{\lmin_i}{\lmax_i}{\R^n}$, then the local step-size can be tuned using the strong convexity and smoothness moduli of $f_i$, instead of the global moduli $\lmin = \min_i \lmin_i$, $\lmax = \max_i \lmax_i$.
This is especially useful because it means that the algorithm can adapt to the \emph{heterogeneity} of the local costs.
\end{remark}

\subsection{Partial participation}\label{subsec:partial-participation}
As discussed in \cref{subsec:design-objectives}, there are different reasons why only a subset of the agents may participate at any given time $k \in \N$.
On the one hand, the heterogeneity of the agents' hardware (CPU, battery, connectivity, ...) may result in different paces of local training. This means that not all agents will conclude their computations at the same time, and the coordinator will receive new results only from those that did.
On the other hand, partial participation can be introduced by design in order to reduce the communication burden. In this case, the coordinator picks a (random) subset of the agents and requests that they send the new results of local training, thus reducing the number of communications exchanged. This scheme is often referred to as \textit{client selection} \cite{grudzien_can_2023}.

The proposed \cref{alg:main-algorithm} allows for any partial participation, be it intrinsic (due to heterogeneity of the agents) or extrinsic (due to client selection). In particular, the only requirement, detailed in \cref{as:stochastic-setup}, is that at any given time an agent $i$ participates in the learning with a given fixed probability $p_i \in (0, 1]$.

\section{Convergence Analysis}\label{sec:convergence}
In this section we prove the convergence of \cref{alg:main-algorithm}. We start by showing that the proposed algorithm can be interpreted as a contractive operator. Then, building on this result we characterize convergence for the different local training choices discussed in \cref{subsec:local-training}.

\subsection{Contractiveness of \cref{alg:main-algorithm}}
The goal of this section is to show that \cref{alg:main-algorithm} can be characterized as a contractive operator, similarly to the PRS from which it was derived.
First of all, we remark that $y_{k+1}$ only depends on $\z_k$, and hence we can represent \cref{alg:main-algorithm} in terms of $\x$ and $\z$ only. In particular, letting $\X : \R^{n N} \times \R^{n N} \to \R^{n N} : (\x_k, \z_k) \mapsto \x_{k+1}$ be the map denoting the local updates, and assuming that \textit{all agents are active}, we can write
\begin{align*}
    \begin{bmatrix}
        \x_{k+1} \\ \z_{k+1}
    \end{bmatrix} =
    \begin{bmatrix}
        \X(\x_k, \z_k) \\
        \z_k + 2 (\x_{k+1} - \prox_{\rho g}(\z_k))
    \end{bmatrix} =
    \T \begin{bmatrix}
        \x_k \\ \z_k
    \end{bmatrix}
\end{align*}
The following \cref{pr:algorithm-contractive} proves that $\T$ is contractive for a suitable choice of local solver.

\begin{proposition}[\cref{alg:main-algorithm} is contractive]\label{pr:algorithm-contractive}
Let the local solvers be $\chi^\Ne$-contractive in the first argument, that is
$$
    \norm{\X(\x, \z_k) - \X(\y, \z_k)} \leq \chi^\Ne \norm{\x - \y}, \quad \forall \x, \y \in \R^n,
$$
and let the PRS applied to~\cref{eq:distributed-problem} be $\zeta$-contractive (cf. \cref{lem:prs}) with fixed point $\bar{\z}$.
Let the parameters $\rho$ and $\Ne$ be chosen such that the following matrix is stable:
$$
    S = \begin{bmatrix}
        \chi^\Ne           & \frac{1 + \chi^\Ne}{\lmin + 1/\rho} \\
        2 \chi^\Ne  & \zeta + \frac{2 \chi^\Ne}{\lmin + 1/\rho}
    \end{bmatrix}.
$$
Then the operator $\T : \R^{n N} \times \R^{n N} \to \R^{n N} \times \R^{n N}$ characterizing \cref{alg:main-algorithm} is contractive. The contraction constant of $\T$ is upper bounded by $\norm{S}$, and its unique fixed point is $[ \bar{\x}^\top, \bar{\z}^\top ]^\top$, where $\bar{\x}$ is the optimal solution to~\cref{eq:distributed-problem}.
\end{proposition}
\begin{proof}
See Appendix~\ref{proof:pr:algorithm-contractive}.
\end{proof}

\smallskip

The contractiveness of \cref{alg:main-algorithm} hinges on the stability of matrix $S \in R^{2 \times 2}$; the following result shows that it is indeed always possible to choose the parameters of \alg that stabilize $S$.

\begin{lemma}[Stabilizing $S$]\label{lem:stable-s}
Given \cref{alg:main-algorithm} with $\chi^\Ne$-contractive local solvers, there always exists a choice of parameters such that $S$ in \cref{pr:algorithm-contractive} is stable.
\end{lemma}
\begin{proof}
See Appendix~\ref{proof:lem:stable-s}.
\end{proof}

We remark that the eigenvalues of $S$ depend in a highly non-linear fashion on the parameters of the algorithm. Thus, our result ensures the existence of a stabilizing choice, but does not prescribe which parameters choice might be the best performing one. It is therefore difficult to provide a clear theoretical dependence of the convergence on \textit{e.g.} $\Ne$.
On the other hand, $S$ is a $2 \times 2$ matrix independently of the problem size. Thus finding a stabilizing choice of the parameters in practice is computationally very cheap, for example with a grid search.

\subsection{Convergence results}\label{subsec:convergence}
\cref{pr:algorithm-contractive} proved that \cref{alg:main-algorithm} can be modeled as a contractive operator when all the agents are active. This fact then would allow us to characterize the convergence of the algorithm by using Banach-Picard theorem (cf. \cref{lem:banach-picard}). However, in this section we are interested in analyzing convergence with \textit{partial participation}, that is, when only a (random) subset of agents is active at each time $k \in \N$.
In addition to partial participation, we are interested in convergence when a further source of stochasticity is present: \textit{inexact local solvers}. In \cref{subsec:local-training} we have indeed discussed two local solvers that, either due to computational constraints or privacy concerns, introduce inexactness during training.

To account for both partial participation and inexact local training, we can write \cref{alg:main-algorithm} as the following stochastic update
\begin{equation}\label{eq:prs-fed-rand}
    \begin{bmatrix}
        x_{i,k+1} \\ z_{i,k+1}
    \end{bmatrix} = 
    \begin{cases}
        \begin{bmatrix}
            \X_i(x_{i,k}, \z_k) \\
            z_{i,k} + 2 (\X_i(x_{i,k}, \z_k) - y_{k+1})
        \end{bmatrix} + e_{i,k} \\ \hspace{5cm} \text{if} \ u_{i,k} = 1 \\
        \begin{bmatrix}
            x_{i,k} \\
            z_{i,k}
        \end{bmatrix} \hspace{4cm} \text{if} \ u_{i,k} = 0
    \end{cases}
\end{equation}
where $u_{i,k} \sim Ber(p_i)$ is the r.v. which equals $1$ when agent $i$ is active at time $k$, $\X_i$ denotes the $i$-th component of the local updates map $\X$, and $e_{i,k}$ is a random additive noise modeling inexact local updates (\textit{e.g.} due to the use of stochastic gradients, or due to inexact $\prox_{\rho h}$ computations). We remark that, to the best of our knowledge, the state of the art in random operator theory does not allow for more general models of $u_{i,k}$, pointing to a promising future research direction.

We can now formalize the stochastic set-up of~\cref{eq:prs-fed-rand} as follows, which will then allow convergence analysis with the tools of \cref{subsec:stochastic-operators}.

\begin{assumption}[Stochastic set-up]\label{as:stochastic-setup}
\begin{enumerate}
    \item[] 
    
    \item The random variables $\{ u_{i,k} \}_{k \in \N}$ are such that $p_i = \mathbb{P}[u_{i,k} = 1] > 0$, for any $i = 1, \ldots, N$.

    \item Let $\e_k = [e_{1,k}^\top, \ldots, e_{N,k}^\top]^\top$ be the random vector collecting all additive errors. We assume that $\{ \e_k \}_{k \in \N}$ are i.i.d. and such that $\ev{\norm{\e_k}} = \nu < \infty$.
    
    \item For every $k \in \N$, $\{ u_{i,k} \}_{i = 1}^N$ and $\e_k$ are independent.
\end{enumerate}
\end{assumption}

\smallskip

We are now ready to prove convergence of \cref{alg:main-algorithm} when both (stochastic) gradient descent and accelerated gradient descent are employed as local solvers.

\begin{proposition}[Convergence with (stochastic) grad. desc.]\label{pr:convergence}
Let the parameters of \cref{alg:main-algorithm} be such that \cref{pr:algorithm-contractive} holds, and let \cref{as:stochastic-setup} be verified.
Then the sequence $\{ [ \x_k^\top, \z_k^\top ]^\top \}_{k \in \N}$ generated by the algorithm satisfies the following bound
\begin{equation}\label{eq:bound-convergence}
    \ev{\norm{\begin{bmatrix}
        \x_k - \bar{\x} \\ \z_k - \bar{\z}
    \end{bmatrix}}} \leq \sqrt{\frac{\pmax}{\pmin}} \left( \sigma^k \norm{\begin{bmatrix}
        \x_0 - \bar{\x} \\ \z_0 - \bar{\z}
    \end{bmatrix}} + \frac{1 - \sigma^k}{1 - \sigma} \nu \right)
\end{equation}
where $\pmin = \min_i p_i$, $\pmax = \max_i p_i$, and $\sigma = \sqrt{1 - \pmin + \pmin \norm{S}^2}$.
\end{proposition}
\begin{proof}
See Appendix~\ref{proof:pr:convergence}.
\end{proof}

\smallskip

\cref{pr:convergence} proves that, for a suitable choice of parameters, \cref{alg:main-algorithm} converges when the agents apply (stochastic) gradient descent during local training.
The following result shows that convergence is achieved also when the agents employ the accelerated gradient descent as a local solver.

\begin{proposition}[Convergence with accelerated gradient]\label{pr:convergence-2}
Let the agents employ the accelerated gradient as local solver. Let the parameters $\rho$ and $\Ne$ be chosen such that the following matrix is stable:
$$
    S' = \begin{bmatrix}
        \chi(\Ne)           & \frac{1 + \chi(\Ne)}{\lmin + 1/\rho} \\
        2 \chi(\Ne)  & \zeta + \frac{2 \chi(\Ne)}{\lmin + 1/\rho}
    \end{bmatrix}
$$
where
$
    \chi(\Ne) := \left( 1 + \frac{\lmax + 1/\rho}{\lmin + 1/\rho} \right) \left( 1 - \sqrt{\frac{\lmin + 1/\rho}{\lmax + 1/\rho}} \right)^\Ne.
$
Then the operator $\T' : \R^{n N} \times \R^{n N} \to \R^{n N} \times \R^{n N}$ characterizing \cref{alg:main-algorithm} (with accelerated gradient) is contractive and, under \cref{as:stochastic-setup}, \cref{pr:convergence} holds.
\end{proposition}
\begin{proof}
See Appendix~\ref{proof:pr:convergence-2}.
\end{proof}

\subsection{Discussion}\label{subsec:discussion}

\subsubsection{The importance of initialization}\label{subsubsec:initialization}
Initializing the local update computation~\cref{eq:local-update-approx} with $w_{i,k}^0 = x_{i,k}$ is fundamental to guarantee contractiveness of \cref{alg:main-algorithm}, and thus its convergence, as detailed in \cref{subsec:convergence}.
Indeed, intuitively this initialization implements a feedback loop on $x_{i,k}$ the local states, which serve as approximations of \cref{eq:prs-fed-x}. Thus, the feedback ensures that asymptotically $x_{i,k}$ converge to the correct update, yielding a solution to the problem.

\subsubsection{\alg avoids client drift}
Recalling the discussion in \cref{sec:introduction}, early uses of local training as a heuristic to reduce communications were subject to \textit{client drift} \cite{grudzien_can_2023}. Intuitively, when the agents perform local training in these early algorithms, they are acting in an open loop fashion, and the local states \textit{drift} towards local solutions (or fixed points) -- which differ from each other due to heterogeneity. The sporadic communication rounds are then not enough to keep the drifting clients on track.
On the other hand, we have the class of algorithms which do not undergo client drift when applying local training, even in the presence of heterogeneous costs/data \cite{zhang_fedpd_2021,mitra_linear_2021,condat_tamuna_2023,alghunaim_local_2023,grudzien_can_2023}, including the proposed \alg.

In particular, \alg avoids client drift owing to the following design choices. First of all, the Peaceman-Rachford splitting serves as the foundation for \alg, with the agents needing to solve \cref{eq:prs-fed-x}. As discussed in \cref{sec:algorithm}, \cref{eq:prs-fed-x} cannot be solved exactly in finite time, thus the agents approximate is with a finite number of local training epochs $\Ne$. But it is crucial to notice that even if $\Ne \to \infty$, the agents \textit{do not} experience client drift -- because this recovers the Peaceman-Rachford, which converges to the exact solution of \cref{eq:distributed-problem}.
The second design choice is the initialization discussed in \cref{subsubsec:initialization}, which closes the loop on the local training steps.
Together, these yield a convergent algorithm which does not experience client drift for any choice of $\Ne \in [1, \infty]$.

\subsubsection{Exact convergence}
Notice that when exact local solvers are employed (hence $\nu = 0$ in \cref{as:stochastic-setup}) then \cref{pr:convergence,pr:convergence-2} prove exact convergence to the fixed point $[\bar{\x}^\top \ \bar{\z}^\top]^\top$, \emph{even in the presence of partial participation}. As a consequence, the proposed algorithm indeed allows for the use of local training without degrading accuracy (objective (i) in \cref{subsec:design-objectives}).

\subsubsection{Asymptotic convergence of \alg}
\cref{pr:convergence} ensures that, despite the presence of additive errors, the algorithm does not diverge. In particular, taking the in \cref{eq:bound-convergence} we get
$$
    \limsup_{k \to \infty} \norm{\begin{bmatrix}
        \x_k - \bar{\x} \\ \z_k - \bar{\z}
    \end{bmatrix}} = \sqrt{\frac{\pmax}{\pmin}} \frac{\nu}{1 - \sigma}.
$$
This shows that asymptotically only the additive errors impact the convergence accuracy (through $\nu$), while the initial condition is irrelevant. This result is on par with stochastic optimization algorithms (\textit{e.g.} SGD), which indeed converge to a neighborhood of the optimal solution, whose radius depends on the magnitude of the additive errors \cite{gorbunov_unified_2020,bastianello_stochastic_2022}.

We can now analyze how to improve accuracy by reducing the value of the asymptotic bound. First of all, we notice that using full participation ($\pmax = \pmin = 1$) reduces both $\sqrt{\pmax / \pmin}$ and $\sigma$.
Moreover, we can choose the parameters of \cref{alg:main-algorithm} ($\Ne$, $\rho$, $\gamma$) to further reduce $\sigma$ by minimizing $\norm{S}$.
Finally, we can reduce $\nu$ through changes in the local solver. For example, if $\nu$ is due to the use of stochastic gradients, then we can replace them with more accurate gradient estimators \cite{gorbunov_unified_2020}.

\subsubsection{Convergence rate tightness}\label{subsubsec:tightness}
Resorting to an operator theoretical convergence analysis allows us to prove convergence in a systematic way, including different types of local solvers. However, there are currently some drawbacks to this proof technique: 1) our simulations suggest that the resulting convergence rate bound is not tight, and thus 2) the analysis does not serve as a precise comparison of the rate w.r.t. the state of the art and centralized counterparts (as done in \cite{mitra_linear_2021}); 3) the theoretical rate does not capture the fact that the best convergence rate is achieved for a small value of $\Ne$ rather than for $\Ne \to \infty$ (which yields the original PRS); 4) the analysis does not give insight into whether using accelerated local solvers can improve convergence. Therefore, future research is needed to derive tighter convergence rate bounds which better model the empirical behavior; this analysis could be done for example through performance estimation \cite{ryu_operator_2020}.

\section{Privacy Analysis}\label{sec:privacy}

In this section, we prove that \cref{alg:main-algorithm}, utilizing the noisy gradient descent~\cref{eq:private-local-update-approx} as local training method, guarantees RDP for every data sample in $\mathcal{D}:=\cup_{i=1}^N \mathcal{D}_i$.
We will derive this result in the case of full agents' participation, which can provide an upper bound to the privacy guarantee with partial participation.
Before proceeding to the analysis, we make the following assumption.

\begin{assumption}[Sensitivity]\label{assu:sensitivity}
For each $i \in \{ 1, \ldots, N \}$, there exists some $L > 0$ such that
\begin{equation*}
    \norm{\nabla f_i^{\mathcal{D}_i}(x) - \nabla f_i^{\mathcal{D}'_i}(x)} \leq \frac{L}{q_i}, \quad \forall x \in \R^n
\end{equation*}
where $\nabla f_i^{\mathcal{D}_i}$ denotes the gradient of $f_i$ induced by the dataset $\mathcal{D}_i$, and $\mathcal{D}_i$ and $\mathcal{D}'_i$ are neighboring datasets of size $q_i$.
\end{assumption}

\cref{assu:sensitivity} quantifies the sensitivity of local gradient queries, which is a well-established concept in the field of machine learning with differential privacy \cite{chourasia2021differential,altschuler_privacy_2022}. It is commonly satisfied in various learning problems, such as those involving $l_2$ regularized logistic loss (where the regularization ensures that \cref{as:costs} is also satisfied).
For more general loss functions, the use of clipped gradients during training guarantees that \cref{assu:sensitivity} is verified \cite{andrew2021differentially}. Specifically, the gradients are clipped using the expression $\nabla \ell_i(x_i; \xi_i^{(j)})\cdot \min \left\{ 1, \frac{L}{2\lVert \nabla \ell_i(x_i; \xi_i^{(j)}) \rVert } \right\}$, where $\xi_i^{(j)} \in \mathcal{D}_i$.

\smallskip

We are now ready to provide our privacy guarantee.

\begin{proposition}[Privacy]\label{pr:privacy}
Suppose \cref{as:costs,assu:sensitivity} hold. Let the agents employ the noisy gradient descent~\cref{eq:private-local-update-approx} with noise variance $\tau^2 < \infty$ as local solver. If the step-size $\gamma < 2/(\lmax+ \rho^{-1})$ and each $x_{i,0}$ is drawn from $\mathcal{N}(0,2\tau^2 I_n/\lmin)$, then \cref{alg:main-algorithm} satisfies $(\lambda,\varepsilon_i)$-R\'enyi differential privacy for agent $i$ with 
$
\varepsilon_i \leq 	\frac{\lambda L^2}{\lmin \tau^2 q_i^2}\left(1-e^{-\lmin\gamma K\Ne/2} \right).
$
Additionally, denoting $\ubar{q} = \min_i q_i$, \cref{alg:main-algorithm} in the worst case satisfies $(\lambda,\varepsilon)$-R\'enyi differential privacy with 
\begin{equation*}
\varepsilon \leq 	\frac{\lambda L^2}{\lmin \tau^2 \ubar{q}^2}\left(1-e^{-\lmin\gamma K\Ne/2} \right).
\end{equation*}

\end{proposition}
\begin{proof}
See Appendix~\ref{proof:pr:privacy}.
\end{proof}


By applying \cref{lem:DP_Conv}, we can convert the RDP guarantee presented in \cref{pr:privacy} into an $(\varepsilon,\delta)$-ADP guarantee, where $\varepsilon$ is given by the following expression:

\begin{equation*}
\varepsilon = \frac{\lambda L^2}{\lmin \tau^2 \ubar{q}^2}\left(1-e^{-\lmin\gamma K\Ne/2} \right) + \frac{\log(1/\delta)}{\lambda-1}, \quad \forall \delta \in (0,1).
\end{equation*}
This result reveals that the privacy loss of \cref{alg:main-algorithm} grows as the number of iterations (local $\Ne$ or global $K$) increases, eventually reaching a \textit{constant upper bound}. This implies that communication efficiency (which encourages choices of $\Ne > 1$) can be achieved without degrading privacy beyond this constant upper bound.
This characterization is based on the dynamics analysis of RDP loss in \cite{chourasia2021differential}, and is arguably tighter than previous approaches based on advanced composition of differential privacy \cite{kairouz2015composition}, privacy amplification by iteration \cite{feldman2018privacy}, and privacy amplification by subsampling \cite{girgis2021shuffled,liu2023privacy} (cf. the discussion in \cite{chourasia2021differential}).

\smallskip

The use of noisy gradient descent as a privacy preserving mechanism of course has an impact on the accuracy of \cref{alg:main-algorithm}. The following result quantifies the inexact convergence of the algorithm for a given privacy threshold.

\begin{corollary}[Accuracy]\label{cor:privacy-utility}
Consider \cref{alg:main-algorithm} with noisy gradient descent~\cref{eq:private-local-update-approx} as local training method, with noise variance $\tau^2 < \infty$ and step-size $\gamma < 2/(\lmax + \rho^{-1})$. Suppose that each $x_{i,0}$ is drawn from $\mathcal{N}(0, 2\tau^2 I_n/\lmin)$, and that \cref{as:costs,as:stochastic-setup} hold. Then at time $K \in \N$ we have
\begin{align}
    \ev{\norm{\begin{bmatrix}
        \x_K - \bar{\x} \\ \z_K - \bar{\z}
    \end{bmatrix}}} &\leq \norm{S}^K \norm{\begin{bmatrix}
        \x_0 - \bar{\x} \\ \z_0 - \bar{\z}
    \end{bmatrix}} + \label{eq:privacy-accuracy} \\ &+ \frac{1 - \norm{S}^K}{1 - \norm{S}} \tau \sqrt{10 n N \gamma} \frac{1 - \chi^\Ne}{1 - \chi} \nonumber
\end{align}
where recall that $\chi = \max\{ |1 - \gamma (\lmin+1/\rho)|, |1 - \gamma (\lmax+1/\rho)| \}$.
\end{corollary}
\begin{proof}
See Appendix~\ref{proof:cor:privacy-utility}.
\end{proof}

\smallskip

To summarize, injecting noise drawn from $\sqrt{2 \gamma} \mathcal{N}(0, \tau^2 \Im_{nN})$ guarantees that \cref{alg:main-algorithm} is $(\lambda,\varepsilon)$-R\'enyi differentially privacy with
\begin{equation*}
\varepsilon \leq 	\frac{\lambda L^2}{\lmin \tau^2 \min_i q_i}^2\left(1-e^{-\lmin\gamma K\Ne/2} \right)
\end{equation*}
and the accuracy is upper bounded by \cref{eq:privacy-accuracy}.
Therefore, the larger the noise variance $\tau^2$ is, the less accurate the algorithm becomes. This captures the well-known trade-off between guaranteeing privacy and achieving accuracy \cite{bassily2014private,chaudhuri2011differentially}.
Recalling the discussion after \cref{pr:privacy}, we remark again that despite this privacy-accuracy trade-off, communication efficiency of \alg does not affect privacy (and vice versa). Indeed, the privacy does not exceed a constant upper bound for any choice of $\Ne$, which can then be freely chosen to achieve communication efficiency without worrying about privacy.

\section{Numerical Results}\label{sec:numerical}
In this section we evaluate the performance of the proposed \alg, and compare it with the alternative algorithms discussed in \cref{subsec:theoretical-comparison}.
We consider a \textit{logistic regression} problem characterized by the local loss functions
$$
    f_i(x) = \frac{1}{q_i} \sum_{h = 1}^{q_i} \log\left( 1 + \exp\left( - b_{i,h} a_{i,h} x \right) \right) + \epsilon r(x)
$$
where $a_{i,h} \in \R^{1 \times n}$ and $b_{i,h} \in \{ -1, 1 \}$, and $r : \R^n \to \R$ is a regularization function.
In particular, we have $N = 100$, $n = 5$, $q_i = 250$, and $\epsilon = 0.5$. The local data points $\{ (a_{i,h}, b_{i,h}) \}_{h = 1}^{q_i}$ are randomly generated, to have a roughly $50-50$ split between the two classes.

In the following, we consider both convex costs with the regularization $r(x) = \norm{x}^2 / 2$, and non-convex costs with $r(x) = \sum_{j = 1}^n [x]_j^2 / (1 + [x]_j^2)$, where $[x]_j$ selects the $j$ component of the vector \cite{alghunaim_local_2023}.

We measure the convergence speed of the algorithms applied to these problems as follows. Let $\tg$ denote the computational time (in unspecified units) required by an agent to evaluate a local gradient, and $\tc$ the computational time for a round of agent-coordinator communication. Then, as metric we use the computational time required to reach $\norm{\sum_{i = 1}^N \nabla f_i(\bar{x})}^2 \leq 10^{-5}$, with $\bar{x} = \sum_{i = 1}^N x_i / N$.

Finally, the simulations are implemented in Python building on top of the \texttt{tvopt} library \cite{bastianello_tvopt_2021}. All the results involving randomness are averaged over $100$ Monte Carlo iterations.

\subsection{Comparison with the state of the art}\label{subsec:comparison-solvers}
We start by comparing the proposed \alg with the alternative algorithms discussed in \cref{subsec:theoretical-comparison}: FedPD \cite{zhang_fedpd_2021}, FedLin \cite{mitra_linear_2021}, TAMUNA \cite{condat_tamuna_2023}, LED \cite{alghunaim_local_2023}, and 5GCS \cite{grudzien_can_2023}.
For all algorithms except TAMUNA \cite{condat_tamuna_2023}, we set the number of local training epochs to $\Ne = 5$; for TAMUNA, the number of local epochs at time $k$ is drawn from a random geometric distribution with mean equals to $\Ne$.
The other parameters of the algorithms (\textit{e.g.} step-size of the local solvers) are tuned to achieve the best performance possible.

In \cref{tab:comparison-cvx-noncvx} we compare the performance of the algorithms in both the convex and nonconvex settings, assigning the time costs $\tg = 1$ and $\tc = 10$. The table also reports the time required by the algorithms at each iteration $k \in \N$.
\begin{table}[!ht]
\centering
\caption{Convergence speed comparison in the convex and nonconvex settings, with $\tg = 1$ and $\tc = 10$.}
\label{tab:comparison-cvx-noncvx}
    \begin{tabular}{cccc}
    \hline
    Algorithm [Ref.] & \thead{Comp. time} & Convex & Nonconvex \\
    \hline
    FedPD \cite{zhang_fedpd_2021} & $(\Ne \tg + \tc) N$ & $70.5$k & $223.5$k \\
    FedLin \cite{mitra_linear_2021} & $((\Ne+1) \tg + 2 \tc) N$ & $15.6$k & $31.2$k \\
    TAMUNA \cite{condat_tamuna_2023} & $(\Ne \tg + \tc) N^\dagger$ & $25.5$k & $-$ \\
    LED \cite{alghunaim_local_2023} & $(\Ne \tg + \tc) N$ & $51$k & $438$k \\
    5GCS \cite{grudzien_can_2023} & $(\Ne \tg + \tc) N$ & $57$k & $39$k \\
    \hline
    \alg [this work] & $(\Ne \tg + \tc) N$ & $\mathbf{13.5}$\textbf{k} & $\mathbf{21}$\textbf{k} \\
    \hline
    \end{tabular}
    \\\vspace{0.1cm}
    $^\dagger$ in mean
\end{table}
As we can see, in this set-up the \alg outperforms the alternatives, reaching the prescribed accuracy in a shorter time. However, this is not always the case, as highlighted by \cref{tab:comparison-tc}, which reports the computation time -- in the convex setting -- for different choices of $\tc$. Indeed, we see that \alg is slower than FedLin \cite{mitra_linear_2021} when communications are computationally cheaper. Interestingly, FedLin achieves the best theoretical rate (as compared to centralized algorithms) \cite{mitra_linear_2021}, but to do so requires two rounds of communications at each step (see \cref{tab:comparison-cvx-noncvx}). Thus, when communication become more expensive, \alg can outperform FedLin, since it uses only one round of communication.
Moreover, we point out that \alg indeed achieves convergence in nonconvex scenarios as well, pointing to interesting future research directions.
\begin{table}[!ht]
\centering
\caption{Convergence speed comparison in the convex setting, for different values of $\tc$, with $\tg = 1$.}
\label{tab:comparison-tc}
    \begin{tabular}{ccccc}
    \hline
    Algorithm [Ref.] & $\tc = 0.1$ & $\tc = 1$ & $\tc = 10$ & $\tc = 100$ \\
    \hline
    FedPD \cite{zhang_fedpd_2021} & $23.97$k & $28.2$k & $70.5$k & $493.5$k \\
    FedLin \cite{mitra_linear_2021} & $\mathbf{3.72}$\textbf{k} & $\mathbf{4.8}$\textbf{k} & $15.6$k & $123.6$k \\
    TAMUNA \cite{condat_tamuna_2023} & $8.67$k & $10.2$k & $25.5$k & $178.5$k \\
    LED \cite{alghunaim_local_2023} & $17.34$k & $20.4$k & $51$k & $357$k \\
    5GCS \cite{grudzien_can_2023} & $19.38$k & $22.8$k & $57$k & $399$k \\
    \hline
    \alg [this work] & $4.59$k & $5.4$k & $\mathbf{13.5}$\textbf{k} & $\mathbf{94.5}$\textbf{k} \\
    \hline
    \end{tabular}
\end{table}

In \cref{tab:comparison-solvers} we compare \alg with the other algorithms that are designed to use different solvers (gradient or accelerated gradient) and/or partial participation. When employing partial participation, only $N/2$ agents (selected uniformly at random) are active at any iteration. We can see that \alg outperforms the alternatives except in the last set-up.
\begin{table}[!ht]
\centering
\caption{Convergence speed comparison in the convex setting, for different solvers and with partial participation, with $\tg = 1$ and $\tc = 10$.}
\label{tab:comparison-solvers}
    \begin{tabular}{ccccc}
    \hline
    Algorithm [Ref.] & Grad. & \thead{Gradient \\ Partial p.} & Acc. grad. & \thead{Acc. grad. \\ Partial p.} \\
    \hline
    TAMUNA \cite{condat_tamuna_2023} & $25.5$k & $53.25$k & $-$ & $-$ \\
    5GCS \cite{grudzien_can_2023} & $57$k & $22.5$k & $57$k & $\mathbf{23.25}$\textbf{k} \\
    \hline
    \textbf{\alg} [this work] & $\mathbf{13.5}$\textbf{k} & $\mathbf{21.75}$\textbf{k} & $\mathbf{15}$\textbf{k} & $28.5$k \\
    \hline
    \end{tabular}
    \\\vspace{0.1cm}
    $-$ indicates that the algorithm is not designed for the scenario
\end{table}
We further notice that partial participation has an opposite effect on the speed of \alg and 5GCS \cite{grudzien_can_2023}, with the former becoming slower as fewer agents participate. This is in line with the theoretical results of \cref{sec:convergence}; indeed, the convergence rate of \alg is upper bounded by $\sigma = \sqrt{1 - \pmin + \pmin \norm{S}^2}$ with $\pmin = \min_i p_i$, and the smaller $\pmin$ (hence, the less participation), the closer this rate gets to $1$. This phenomenon is further explored in \cref{subsec:partial-participation}.
Another observation is that using gradient descent as local solver in \alg yields faster convergence than accelerated gradient. A possible explanation is that the overall convergence is not dominated by the convergence rate of the local solver. For example, the convergence may be dominated by $\zeta$, the convergence rate of the exact PRS (cf. \cref{pr:algorithm-contractive}). Additional theoretical research is needed to properly analyze this phenomenon and its generality (see also \cref{subsubsec:tightness}).

We conclude comparing in \cref{tab:comparison-larger} the performance of the different algorithms for a problem of much larger size $n = 100$ (rather than $n = 5$). The time costs $\tc$ and $\tg$ were scaled to account for the larger size of the problem.
\begin{table}[!ht]
\centering
\caption{Convergence speed comparison in the convex setting with $n = 100$, for different values of $\tc$, with $\tg = 20$.}
\label{tab:comparison-larger}
    \begin{tabular}{ccccc}
    \hline
    Algorithm [Ref.] & $\tc = 2$ & $\tc = 20$ & $\tc = 200$ & $\tc = 2000$ \\
    \hline
    FedPD \cite{zhang_fedpd_2021} & $693.6$k & $816$k & $2.04$M & $14.28$M \\
    FedLin \cite{mitra_linear_2021} & $\mathbf{86.8}$\textbf{k} & $\mathbf{112}$\textbf{k} & $364$k & $2.884$M \\
    TAMUNA \cite{condat_tamuna_2023} & $244.8$k & $288$k & $720$k & $5.04$M \\
    LED \cite{alghunaim_local_2023} & $346.8$k & $408$k & $1.02$M & $7.14$M \\
    5GCS \cite{grudzien_can_2023} & $306$k & $360$k & $900$k & $6.3$M \\
    \hline
    \alg [this work] & $102$k & $120$k & $\mathbf{300}$\textbf{k} & $\mathbf{2.1}$\textbf{M} \\
    \hline
    \end{tabular}
\end{table}
We notice that the same relative performance as in \cref{tab:comparison-tc} can be observed when applying the algorithms to this larger problem, which testifies to the good scalability properties of \alg.

\subsection{\alg with partial participation}\label{subsec:partial-participation}
In the previous section we noticed that using partial participation ($50\%$ of agents) may slow down the convergence of \alg. In \cref{tab:partial-participation} we report the computation time when different percentages of the agents are active at each iteration.
\begin{table}[!ht]
\centering
\caption{Convergence rate of \alg with partial participation.}
\label{tab:partial-participation}
    \begin{tabular}{ccc}
    \hline
    $\%$ active agents & Convergence rate \\
    \hline
    $40\%$  &   $22.8$k \\
    $50\%$  &   $21.75$k \\
    $60\%$  &   $19.8$k \\
    $70\%$  &   $19.95$k \\
    $80\%$  &   $16.8$k \\
    $90\%$  &   $16.2$k \\
    $100\%$ &   $\mathbf{13.5}$\textbf{k} \\
    \hline
    \end{tabular}
\end{table}
We can see that indeed the convergence tends to speed up as more agents participate, although the trend is not strictly monotonic. This is in part due to the randomness of the active agents selection, and in part because the computation time changes with the number of active agents.

\subsection{Privacy-preserving \alg}\label{subsec:privacy-results}
We turn now to studying the performance of \alg when the privacy-preserving noisy gradient descent is used as a local solver.
As discussed in \cref{sec:privacy}, the use of such solver results in a decrease in accuracy of the learned model, with the asymptotic error being bounded by a non-zero value (cf. \cref{cor:privacy-utility}).
In \cref{tab:noisy-gradient-variance} we report the empirical value of such asymptotic error when different values of the noise variance $\tau$ are used. As expected, the larger the variance is the larger the asymptotic error; and considering that we have $100$ agents, each injecting noise, this error can quickly grow. Therefore, setting the parameters of \alg should strike the right balance between privacy and accuracy.
\begin{table}[!ht]
\centering
\caption{Asymptotic error of \alg when using privacy-preserving noisy gradient with different variances, in the convex setting with $\tg = 1$ and $\tc = 10$.}
\label{tab:noisy-gradient-variance}
    \begin{tabular}{ccc}
    \hline
    Noise variance $\tau$ & Asymptotic err. \\
    \hline
    $10^{-6}$   &   $2.583 \times 10^{-2}$ \\
    $10^{-5}$   &   $7.369 \times 10^{-2}$ \\
    $10^{-4}$   &   $2.524 \times 10^{-1}$ \\
    $10^{-3}$   &   $7.920 \times 10^{-1}$ \\
    $10^{-2}$   &   $2.565$ \\
    $10^{-1}$   &   $8.116$ \\
    $1$         &   $26.055$ \\
    $10$        &   $85.273$ \\
    \hline
    \end{tabular}
\end{table}

\subsection{Tuning the parameters of \alg}\label{subsec:tuning-parameters}
The results of the previous section were obtained by setting $\Ne = 5$ in \alg, and tuning the penalty $\rho$ to achieve an optimal convergence rate. In this section we further characterize the effect of both $\rho$ and $\Ne$ on \alg's convergence.

\cref{tab:rho} reports the computation time for three choices of $\rho$, and we can observe that impact of this parameter on the convergence speed is not monotone. Indeed, values around $\rho = 1$ show empirically the best performance, similarly to the PRS from which \alg is derived, see \textit{e.g.} \cite[Figures~1,~2]{ryu_operator_2020}.

\begin{table}[!ht]
\centering
\caption{Performance of \alg for different values of $\rho$, in the convex setting with $\tg = 1$ and $\tc = 10$.}
\label{tab:rho}
    \begin{tabular}{c|ccc}
    \hline
    $\rho$ & $0.1$ & $1$ & $10$ \\
    Comp. time    &   $184.5$k & $58.5$k & $243$k \\
    \hline
    \end{tabular}
\end{table}

\cref{tab:Ne} compares the performance of \alg for different numbers of local training epochs $\Ne$, and with different choices of $\tc$.
\begin{table}[!ht]
\centering
\caption{Performance of \alg for different values of $\Ne$, in the convex setting with $\tg = 1$ and varying $\tc$.}
\label{tab:Ne}
    \begin{tabular}{ccccc}
    \hline
    $\Ne$ & $\tc = 0.1$ & $\tc = 1$ & $\tc = 10$ & $\tc = 100$ \\
    \hline
    $1$ & $3.19$k & $5.8$k & $31.9$k & $292.9$k \\
    $2$ & $\mathbf{3.15}$\textbf{k} & $\mathbf{4.5}$\textbf{k} & $18$k & $153$k \\
    $5$ & $4.59$k & $5.4$k & $\mathbf{13.5}$\textbf{k} & $94.5$k \\
    $8$ & $6.48$k & $7.2$k & $14.4$k & $\mathbf{86.4}$\textbf{k} \\
    $10$ & $8.08$k & $8.8$k & $16$k & $88$k \\
    $20$ & $16.08$k & $16.8$k & $24$k & $96$k \\
    \hline
    \end{tabular}
\end{table}
Interestingly, the convergence speed is not a monotonically decreasing function of $\Ne$. Rather, the best performance is achieved for intermediate (finite) values of $\Ne$, despite the fact that \alg is built on the blueprint of PRS which has $\Ne \to \infty$. Providing a theoretical understanding of this observation is an interesting future direction.
We remark that a similar phenomenon is observed for the ProxSkip algorithm \cite{mishchenko_proxskip_2022}, which at each iteration performs a round of communications with some probability. The authors show that the convergence rate is the same for all values of communication probability in $[1/\sqrt{\lmin \lmax}, 1]$; thus, the communication burden can be reduced without impacting the speed of convergence.
Finally, we notice that the optimal choice of $\Ne$ changes based on the relative time cost of communications and gradient evaluations. In particular, as $\tc$ grows larger, so does the optimal $\Ne$, confirming that more local training is the solution to expensive communications.

\appendices

\section{Proofs of \Cref{sec:algorithm}}\label{proof:sec:algorithm}

\subsection{Proof of \cref{lem:prox-g}}

By \cref{def:proximal}, to compute the proximal of $g$ we need to solve:
\begin{align*}
    &\prox_{\rho g}(\z) = \argmin_{\x \in \R^{n N}} \left\{ \iota_\C(\x) + h(x_1) + \frac{1}{2\rho} \norm{\x - \z}^2 \right\} \\
    &\overset{(i)}{=} \1 \otimes \argmin_{x \in \R^n} \left\{ h(x) + \frac{1}{2\rho} \sum\nolimits_i \norm{x - z_i}^2 \right\} \\
    &\overset{(ii)}{=} \1 \otimes \argmin_{x \in \R^n} \left\{ h(x) + N \frac{1}{2\rho} \norm{x - \frac{1}{N} \sum\nolimits_i z_i}^2 \right\} \\
    &= \1 \otimes \prox_{\rho h / N}\left(\frac{1}{N} \sum\nolimits_i z_i \right)
\end{align*}
where (i) follows by using $\x = \1 \otimes x$ to enforce the consensus constraints, (ii) follows by manipulations of the quadratic regularization and adding and summing terms that do not depend on $x$. \qed

\section{Proofs of \Cref{sec:convergence}}\label{proof:sec:convergence}

\subsection{Proof of \cref{pr:algorithm-contractive}}\label{proof:pr:algorithm-contractive}
\textit{1) Fixed point}: We start by showing that $[ \bar{\x}^\top, \bar{\z}^\top ]^\top$ is indeed a fixed point for $\T$.
By \cref{lem:prs}, the PRS applied to~\cref{eq:distributed-problem} is contractive, since $f \in \scs{\lmin}{\lmax}{\R^{n N}}$, and its fixed point $\bar{\z}$ is such that
$
    \bar{\x} = \prox_{\rho g}(\bar{\z}) = \prox_{\rho f}(2\bar{\x} - \bar{\z}).
$
Therefore, defining
$
    d(\x; \z) = f(\x) + \frac{1}{2 \rho} \norm{\x - \z}^2
$
we have
\begin{align*}
    \bar{\x} &= \argmin_\x d(\x; \refl_{\rho g}(\bar{\z})) \\
    &= \argmin_\x d(\x; 2 \prox_{\rho g}(\bar{\z}) - \bar{\z}).
\end{align*}
But this implies that, if we initialize the local updates of \cref{alg:main-algorithm} at $\bar{\x}, \bar{\z}$, then they will return $\bar{\x}$ itself, since it is the stationary point of $d(\cdot; \refl_{\rho g}(\bar{\z}))$.
This fact, in addition to the fact that $\bar{\z}$ is the fixed point of the PRS~\cref{eq:prs-fed}, proves that indeed
$
    [ \bar{\x}^\top, \bar{\z}^\top ]^\top = \T [ \bar{\x}^\top, \bar{\z}^\top ]^\top.
$

\textit{2) Contractiveness}: We proceed now by proving that $\T$ is contractive, and characterizing its contraction rate. We do this by deriving bounds to the distances $\norm{\x_{k+1} - \bar{\x}}$ and $\norm{\z_{k+1} - \bar{\z}}$.

We start by providing the following auxiliary result. By the discussion above we know that $\bar{\x} = \argmin_\x d(\x; \refl_{\rho g}(\bar{\z}))$, and we can similarly define the exact local update~\cref{eq:prs-fed-x} by $\bar{\x}_{k+1} = \argmin_\x d(\x; \refl_{\rho g}(\z_k))$. $\bar{\x}_{k+1}$ thus represents the exact local update when applying PRS, which instead is approximated by $\Ne$ local training steps in \alg.
But since $d$ is $(\lmin + 1/\rho)$-strongly convex in $\x$, we can apply the implicit function theorem \cite[Theorem~1B.1]{doncev_implicit_2014} to prove that
\begin{align}
    \norm{\bar{\x}_{k+1} - \bar{\x}} &\leq \frac{1}{\lmin + 1/\rho} \norm{\refl_{\rho g}(\z_k) - \refl_{\rho g}(\bar{\z})} \nonumber \\
    &\overset{(i)}{\leq} \frac{1}{\lmin + 1/\rho} \norm{\z_k - \bar{\z}} \label{eq:implicit-thm-bound}
\end{align}
where (i) holds by the fact that the reflective operator is non-expansive \cite{bauschke_firmly_2012}.
With this result in place, we can now derive the following chain of inequalities:
\begin{align}
    &\norm{\x_{k+1} - \bar{\x}} \overset{(i)}{\leq} \norm{\x_{k+1} - \bar{\x}_{k+1}} + \norm{\bar{\x}_{k+1} - \bar{\x}} \nonumber \\
    &\qquad \overset{(ii)}{\leq} \norm{\x_{k+1} - \bar{\x}_{k+1}} + \frac{1}{\lmin + 1/\rho} \norm{\z_k - \bar{\z}} \nonumber \\
    &\qquad \overset{(iii)}{\leq} \chi^\Ne \norm{\x_k - \bar{\x}_{k+1}} + \frac{1}{\lmin + 1/\rho} \norm{\z_k - \bar{\z}} \label{eq:x-bound-temp}
\end{align}
where (i) holds by triangle inequality, (ii) by~\cref{eq:implicit-thm-bound}, and (iii) by contractiveness of the local solvers.
By again using the triangle inequality and~\cref{eq:implicit-thm-bound} we have
\begin{align}
    \norm{\x_k - \bar{\x}_{k+1}} &\leq \norm{\x_k - \bar{\x}} + \norm{\bar{\x} - \bar{\x}_{k+1}} \nonumber \\
    &\leq \norm{\x_k - \bar{\x}} + \frac{1}{\lmin + 1/\rho} \norm{\z_k - \bar{\z}}, \label{eq:x-bound-1}
\end{align}
and substituting into~\cref{eq:x-bound-temp} we get
\begin{equation}\label{eq:x-bound}
    \norm{\x_{k+1} - \bar{\x}} \leq \chi^\Ne \norm{\x_k - \bar{\x}} + \frac{1 + \chi^\Ne}{\lmin + 1/\rho} \norm{\z_k - \bar{\z}}.
\end{equation}

We turn now to bounding $\norm{\z_{k+1} - \bar{\z}}$. By the updates characterizing \cref{alg:main-algorithm} we can write
\begin{align*}
    \z_{k+1} &= \z_k + 2 \left( \x_{k+1} - \prox_{\rho g}(\z_k) \right) \\
    &= \z_k + 2 \left( \bar{\x}_{k+1} - \prox_{\rho g}(\z_k) \right) + 2 (\x_{k+1} - \bar{\x}_{k+1}) \\
    &= \T_{\mathrm{PR}} \z_k + 2 (\x_{k+1} - \bar{\x}_{k+1})
\end{align*}
where recall that $\bar{\x}_{k+1}$ denotes the exact local updates, and $\T_{\mathrm{PR}}$ denotes the exact PRS.

We can then derive the following bound:
\begin{align}
    &\norm{\z_{k+1} - \bar{\z}} \overset{(i)}{\leq} \norm{\T_{\mathrm{PR}} \z_k - \bar{\z}} + 2 \norm{\x_{k+1} - \bar{\x}_{k+1}} \nonumber \\
    &\qquad \overset{(ii)}{\leq} \zeta \norm{\z_k - \bar{\z}} + 2 \norm{\x_{k+1} - \bar{\x}_{k+1}} \nonumber \\
    &\qquad \overset{(iii)}{\leq} \left( \zeta + \frac{2 \chi^\Ne}{\lmin + 1/\rho} \right) \norm{\z_k - \bar{\z}} + 2 \chi^\Ne \norm{\x_k - \bar{\x}} \label{eq:z-bound}
\end{align}
where (i) holds by triangle inequality, (ii) by the PRS operator $\T_{\mathrm{PR}}$ being contractive by \cref{lem:prs}, and (iii) by using the following bound
\begin{align*}
    &\norm{\x_{k+1} - \bar{\x}_{k+1}} \overset{(iv)}{\leq} \chi^\Ne \norm{\x_k - \bar{\x}_{k+1}} \\
    &\qquad \overset{(v)}{\leq} \chi^\Ne \left( \norm{\x_k - \bar{\x}} + \frac{1}{\lmin + 1/\rho} \norm{\z_k - \bar{\z}} \right),
\end{align*}
where (iv) holds by contractiveness of the local solvers, and (v) by \cref{eq:x-bound-1}.

Finally, putting~\cref{eq:x-bound,eq:z-bound} together yields
$
    \begin{bmatrix}
        \norm{\x_{k+1} - \bar{\x}} \\ \norm{\z_{k+1} - \bar{\z}}
    \end{bmatrix} \leq S
    \begin{bmatrix}
        \norm{\x_k - \bar{\x}} \\ \norm{\z_k - \bar{\z}}
    \end{bmatrix}
$
and taking the square norm on both sides
\begin{align*}
    &\norm{\x_{k+1} - \bar{\x}}^2 + \norm{\z_{k+1} - \bar{\z}}^2 \leq \norm{S
    \begin{bmatrix}
        \norm{\x_k - \bar{\x}} \\ \norm{\z_k - \bar{\z}}
    \end{bmatrix}}^2 \\
    &\qquad \leq \norm{S}^2 \left( \norm{\x_k - \bar{\x}}^2 + \norm{\z_k - \bar{\z}}^2 \right).
\end{align*}
By \cref{def:contractive-operator} this proves that \cref{alg:main-algorithm} is contractive. \qed

\subsection{Proof of \cref{lem:stable-s}}\label{proof:lem:stable-s}
We observe that the eigenvalues of $S$ are $(1/2) t \pm \sqrt{v}$, where
$
    t := \zeta + \chi^\Ne \left(1 + \frac{2}{\lmin + 1/\rho}\right),
$
$
    v:= t^2 - 4 \chi^\Ne \left(\zeta - \frac{2}{\lmin + 1/\rho}\right);
$
moreover, these eigenvalues are real since $v > 0$ by non-negativity of $\chi$, $\zeta$, and $\frac{2}{\lmin + 1/\rho}$.
The goal then is to show that there exists at least a choice of parameters for which $-1 < (1/2) t \pm \sqrt{v} < 1$. With some simple algebra, it is possible to see that stability is achieved if the inequalities  $(t + 2)^2 > v$ and $(t - 2)^2 < v$ are verified. The first one is always satisfied by non-negativity, while the second one requires that $(1 - \zeta) (1 - \chi^\Ne) < 4 \chi^\Ne / (\lmin + 1/\rho)$.
Picking $\Ne = 1$, this inequality is satisfied if $\chi > \frac{1 - \zeta}{1 - \zeta + 4 / (\lmin + 1/\rho)}$, which in turn is satisfied since $\chi$ can vary in $(0, 1)$ by the choice of \textit{e.g.} the step-size $\gamma$ for gradient descent. \qed

\subsection{Proof of \cref{pr:convergence}}\label{proof:pr:convergence}
If \cref{pr:algorithm-contractive} is verified, then the algorithm is contractive, and under \cref{as:stochastic-setup} we can apply the stochastic Banach-Picard of \cref{lem:stochastic-banach-picard}. \qed

\subsection{Proof of \cref{pr:convergence-2}}\label{proof:pr:convergence-2}
We start by proving the following auxiliary lemma.

\begin{lemma}[Contraction of accelerated gradient]\label{lem:contraction-accelerated}
Let $f \in \scs{\lmin}{\lmax}{\R^n}$, then the accelerated gradient descent
\begin{align*}
    \x^{\ell+1} &= \w^\ell - \frac{1}{\lmax} \nabla f(\w^\ell) \\
    \w^{\ell+1} &= \x^{\ell+1} + \frac{\sqrt{\lmax} - \sqrt{\lmin}}{\sqrt{\lmax} + \sqrt{\lmin}} (\x^{\ell+1} -\x^\ell)
\end{align*}
with initial conditions $\x^0 = \w^0$ is such that
$$
    \norm{\x^\ell - \bar{\x}} \leq \left( 1 + \lmax / \lmin \right) \left( 1 - \sqrt{\lmin / \lmax} \right)^\ell \norm{\x^0 - \bar{\x}}
$$
where $\bar{\x}$ is the unique minimum of $f$.
Therefore,
$
    \ell > \frac{\log(1 + \lmax / \lmin)}{| \log(1 - \sqrt{\lmin / \lmax}) |}
$
is a sufficient condition for the accelerated gradient to be contractive.
\end{lemma}
\begin{proof}
Define
$
    \phi^\ell = f(\x^\ell) - f(\bar{\x}) + \frac{\lmin}{2} \norm{\w^\ell - \bar{\x}}^2,
$
then by \cite[Theorem~4.14]{daspremont_acceleration_2021} we have that $\phi^{\ell+1} \leq \zeta \phi^\ell$ with $\zeta = 1 - \sqrt{\lmin / \lmax}$.
Now, by smoothness of $f$ and the fact that $\nabla f(\bar{\x})$ we have $f(\x^0) - f(\bar{\x}) \leq (\lmax / 2) \norm{\x^0 - \bar{\x}}^2$, which implies
$
    \phi^0 \leq \frac{\lmax + \lmin}{2} \norm{\w^0 - \bar{\x}}^2.
$
Therefore, using $\phi^{\ell+1} \leq \zeta \phi^\ell$ recursively we get
$
    \phi^\ell \leq \zeta^\ell \phi^0 \leq \frac{\lmax + \lmin}{2} \norm{\w^0 - \bar{\x}}^2
$
and, together with the fact $(\lmin / 2) \norm{\w^\ell - \bar{\x}}^2 \leq \phi^\ell$ (which holds by definition of $\phi^\ell$), the thesis follows.
The algorithm is then contractive w.r.t. $\x$ whenever a suitably large $\ell$ is chosen.
\end{proof}

\smallskip



The contractiveness of \cref{alg:main-algorithm} with accelerated gradient as local solver follows along the lines of the proof of \cref{pr:algorithm-contractive}, with $\chi^\Ne$ being replaced by $\chi(\Ne)$. Indeed, by \cref{lem:contraction-accelerated} above, the local update map $\X'$ verifies
$$
    \norm{\X(\x, \z_k) - \X(\y, \z_k)} \leq \chi(\Ne) \norm{\x - \y}, \quad \forall \x, \y \in \R^n.
$$
The convergence under \cref{as:stochastic-setup} is thus another consequence of \cref{lem:stochastic-banach-picard}. \qed

\section{Proofs of \Cref{sec:privacy}}\label{proof:sec:privacy}

\subsection{Proof of \cref{pr:privacy}}\label{proof:pr:privacy}
We denote by ${\x}_{k}:= [{x}^\top_{1,k}, \ldots, {x}^\top_{N,k}]^\top$ and ${\x}_{k}':= [{x}'^\top_{1,k}, \ldots, {x}'^\top_{N,k}]^\top$ the iterates produced by all the agents at the $k$'s iteration based on the neighboring datasets $\mathcal{D}$ and $\mathcal{D}'$, respectively. Furthermore, let $\boldsymbol{ X}_{k}$ and $\boldsymbol{ X}'_{k}$ be the corresponding random variables that model ${\x}_{k}$ and ${\x}_{k}'$. 
We abuse the notation to also denote their probability distributions by $\boldsymbol{ X}_{k}$ and $\boldsymbol{ X}'_{k}$. We are interested in the worst-case R\'enyi divergence between the output distributions based on $\mathcal{D}$ and $\mathcal{D}'$, i.e.,
$
{\rm D}_\lambda(\boldsymbol{ X}_{K} ||\boldsymbol{ X}'_{K})
$.
Since each agent generates the Gaussian noise independently, $x_{i,k}$ is independent between different $i$. In addition, two neighboring datasets $\mathcal{D}$ and $\mathcal{D}'$ only have at most one data point that is different. Without loss of generality, we assume this different data point is possessed by agent $i$, which has local dataset size of $q_i$. Then, we obtain from \cite[Theorem 27]{van2014renyi} that
$
{\rm D}_\lambda(\boldsymbol{ X}_{ K} ||\boldsymbol{ X}'_{K}) = {\rm D}_\lambda(\boldsymbol{ X}_{i,K} ||\boldsymbol{ X}'_{i,K}).
$

Next, we follow the approach in \cite{chourasia2021differential} to quantify the privacy loss during the local training process at the $k$'s iteration, that is, ${\rm D}_\lambda(\boldsymbol{ X}_{i,k+1} ||\boldsymbol{ X}'_{i,k+1}) - {\rm D}_\lambda(\boldsymbol{ X}_{i,k} ||\boldsymbol{ X}'_{i,k})$. Recall that the local training update at the $\ell$-th step reads
\begin{equation}\label{eq:noisy_GD}
	\begin{split}
		& w_{i,k}^{\ell+1} = w_{i,k}^\ell - \gamma \nabla d_{i,k}(w_{i,k}^\ell) + t_{i,k}^\ell 
	\end{split}
\end{equation}
where $	w_{i,k}^0= x_{i,k}$ and $t_{i,k}^\ell \sim \sqrt{2\gamma} \mathcal{N}(0,\tau^2{I}_n)$. Note that  $d_{i,k}$ is $(\lmin+\rho^{-1})$-strongly convex and $(\lmax+\rho^{-1})$-smooth.
Let $\boldsymbol{W}_{i,k}^{\gamma \ell}$ and $\boldsymbol{ W}_{i,k}'^{\gamma \ell}$ be the probability distributions that model $w_{i,k}^{\ell}$ and $w_{i,k}'^{\ell}$. During the time interval $\gamma \ell<t<\gamma(\ell+1)$, the updates in~\cref{eq:noisy_GD} on the datasets $\mathcal{D}$ and $\mathcal{D'}$ are modeled using the following stochastic processes.
\begin{equation}\label{eq:diffusion_process}
	\left\{
	\begin{aligned}
		\boldsymbol{ W}_{ i,k}^t  &=   	\boldsymbol{ W}_{ i,k}^{\gamma \ell} - \gamma U_1(\boldsymbol{ W}_{ i,k}^{\gamma \ell}) - (t-\gamma \ell)  U_2(\boldsymbol{ W}_{ i,k}^{\gamma \ell}) \\
		& \quad + \sqrt{2(t-\gamma \ell)\tau^2 } \boldsymbol{G}^\ell \\
	  	\boldsymbol{ W}_{ i,k}'^t &=   	\boldsymbol{ W}_{ i,k}'^{\gamma \ell} - \gamma U_1(\boldsymbol{ W}_{ i,k}'^{\gamma \ell}) + (t-\gamma \ell)  U_2(\boldsymbol{ W}_{ i,k}'^{\gamma \ell})\\
	  	& \quad + \sqrt{2(t-\gamma \ell)\tau^2 } \boldsymbol{G}^\ell
	\end{aligned}
	\right.
\end{equation}
where $\boldsymbol{G}^\ell \sim \mathcal{N}(0,{I}_n) $, $U_1(w_{i,k})= ( \nabla d_{i,k}^{\mathcal{D}}(w_{i,k}) + \nabla d_{i,k}^{\mathcal{D}'}(w_{i,k})   )/2$ and  $U_2(w_{i,k})= ( \nabla d_{i,k}^{\mathcal{D}}(w_{i,k}) - \nabla d_{i,k}^{\mathcal{D}'}(w_{i,k})   )/2$. Conditioned on observing $w_{i,k}^{\ell}$, $w_{i,k}'^{\ell}$ and $v_{i,k}$, the diffusion processes in~\cref{eq:diffusion_process} have the following stochastic differential equations
\begin{equation*}
	\begin{split}
		d\boldsymbol{ W}_{ i,k}^t & = - U_2(w_{i,k}^{\ell})dt + \sqrt{2\tau^2} d \boldsymbol{Z}^t \\
			d\boldsymbol{ W}_{ i,k}'^t & = - U_2(w_{i,k}'^{\ell})dt + \sqrt{2\tau^2} d \boldsymbol{Z}^t
	\end{split}
\end{equation*}
where $d \boldsymbol{Z}^t \sim \sqrt{dt}\mathcal{N}(0,I_n)$ describes the Wiener process. 

For the training process at $k$'s iteration, we use Assumption \ref{assu:sensitivity} and a similar argument with (67) in \cite[Theorem 2]{chourasia2021differential} to obtain
\begin{equation*}
\begin{split}
& \frac{{\rm D}_\lambda(\boldsymbol{ W}^{\Ne}_{i,k} ||\boldsymbol{ W}'^{\Ne}_{i,k})}{\lambda} - \frac{L^2}{  \lmin \tau^2 q_i^2}  \\
&\leq   \left(  \frac{ {\rm D}_{\lambda_{k}}(\boldsymbol{ W}^0_{i,k} ||\boldsymbol{ W}'^0_{i,k}) }{\lambda_k} - \frac{L^2}{ \lmin \tau^2 q_i^2} \right) e^{-\lmin\gamma \Ne/2} 
\end{split}
\end{equation*}
for some $\lambda_k >1$. Because of lines 5 and 9 in \cref{alg:main-algorithm}, there holds ${\rm D}_{\lambda_k}(\boldsymbol{ W}_{i,k}^{0}|| \boldsymbol{ W}_{i,k}'^{0} ) = {\rm D}_{\lambda_k}(\boldsymbol{ X}_{i,k} ||\boldsymbol{ X}'_{i,k})$ and ${\rm D}_\lambda(\boldsymbol{ W}_{i,k}^{\Ne}|| \boldsymbol{ W}_{i,k}'^{\Ne} ) = {\rm D}_\lambda(\boldsymbol{ X}_{i,k+1} ||\boldsymbol{ X}'_{i,k+1})$. By iterating over $k=0,\dots, K-1$, we obtain 
\begin{equation*}
\begin{split}
& \frac{{\rm D}_\lambda(\boldsymbol{ X}_{i,K} ||\boldsymbol{ X}'_{i,K})}{\lambda} - \frac{L^2}{ \lmin \tau^2 q_i^2}  \\
&\leq   \left(  \frac{ {\rm D}_{\lambda_{0}}(\boldsymbol{ X}_{i,0} ||\boldsymbol{ X}'_{i,0}) }{\lambda_0} - \frac{L^2}{ \lmin \tau^2 q_i^2} \right) e^{-\lmin\gamma K \Ne/2} 
\end{split}
\end{equation*}
for some $\lambda_0$.
Recall that $x_{i,0}=x'_{i,0}$, we arrive at

\begin{equation*}
\begin{split}
& {\rm D}_\lambda(\boldsymbol{ X}_{K} ||\boldsymbol{ X}'_{K})  =	{\rm D}_\lambda(\boldsymbol{ X}_{i,K} ||\boldsymbol{ X}'_{i,K}) \\
&\leq  \frac{\lambda L^2}{\lmin \tau^2 q_i^2}\left(1-e^{-\lmin\gamma K\Ne/2} \right),
\end{split}
\end{equation*}
which quantifies the differential privacy of agent $i$. Moreover, we can characterize the global differential privacy with the worst case bound
$
{\rm D}_\lambda(\boldsymbol{ X}_{K} ||\boldsymbol{ X}'_{K})  =	{\rm D}_\lambda(\boldsymbol{ X}_{i,K} ||\boldsymbol{ X}'_{i,K}) \leq  \frac{\lambda L^2}{\lmin \tau^2 \ubar{q}^2}\left(1-e^{-\frac{\lmin\gamma K\Ne}{2}} \right).
$

\subsection{Proof of \cref{cor:privacy-utility}}\label{proof:cor:privacy-utility}
The result follows as a consequence of \cref{pr:convergence} by showing that \cref{alg:main-algorithm} with noisy gradient descent as local training method can be written as~\cref{eq:prs-fed-rand}.
Denote by $\mathcal{X}(\x_k, \z_k)$ and $\hat{\mathcal{X}}(\x_k, \z_k)$ the output of the local training step when gradient, respectively noisy gradient, is used. Then we can characterize \cref{alg:main-algorithm} by the following update
\begin{align*}
    \begin{bmatrix}
        \x_{k+1} \\ \z_{k+1}
    \end{bmatrix} =
    \begin{bmatrix}
        \X(\x_k, \z_k) \\
        \z_k + 2 (\X(\x_k, \z_k) - \prox_{\rho g}(\z_k))
    \end{bmatrix} + \e_k
\end{align*}
where
$
    \e_k := \begin{bmatrix} 1 \\ 2 \end{bmatrix} \otimes \left( \hat{\X}(\x_k, \z_k) - \X(\x_k, \z_k) \right);
$
the goal now is to bound $\ev{\norm{\e_k}}$.

\smallskip

Let us define the gradient descent operator applied during the local training of iteration $k$ as
$
    \mathcal{G}_k \w = \w - \gamma \left( \nabla f(\w) + (\w - \vv_k) \rho \right).
$
By the choice of step-size $\gamma < 2/(\lmax + \rho^{-1})$ we know that $\mathcal{G}_k$ is $\chi$-contractive, $\chi = \max\{ |1 - \gamma (\lmin+1/\rho)|, |1 - \gamma (\lmax+1/\rho)| \}$ (cf. \cref{lem:gradient-descent}).
We can then denote the trajectories generated by the gradient descent and noisy gradient descent, for $\ell \in \{ 0, \ldots, \Ne-1 \}$, by
$
    \w_k^{\ell+1} = \mathcal{G}_k \w_k^\ell \quad \text{and} \quad \hat{\w}_k^{\ell+1} = \mathcal{G}_k \hat{\w}_k^\ell + \tv_k^\ell
$
with $\tv_k^\ell \sim \sqrt{2 \gamma} \mathcal{N}(0, \tau^2 \Im_{n N})$. Therefore we have $\w_k^\Ne = \X(\x_k, \z_k)$ and $\hat{\w}_k^\Ne = \hat{\X}(\x_k, \z_k)$, as well as $\w_k^0 = \hat{\w}_k^0 = \x_k$.

We have therefore the following chain of inequalities
\begin{align}
    \Big\lVert \hat{\X} & (\x_k, \z_k) - \X(\x_k, \z_k) \Big\lVert = \norm{\hat{\w}_k^\Ne - \w_k^\Ne} \nonumber \\
    &\qquad = \norm{\mathcal{G}_k \hat{\w}_k^{\Ne-1} + \tv_k^{\Ne-1} - \mathcal{G}_k \w_k^{\Ne-1}} \nonumber \\
    &\qquad \overset{(i)}{\leq} \norm{\mathcal{G}_k \hat{\w}_k^{\Ne-1} - \mathcal{G}_k \w_k^{\Ne-1}} + \norm{\tv_k^{\Ne-1}} \nonumber \\
    &\qquad \overset{(ii)}{\leq} \chi \norm{\hat{\w}_k^{\Ne-1} - \w_k^{\Ne-1}} + \norm{\tv_k^{\Ne-1}} \nonumber \\
    &\qquad \overset{(iii)}{\leq}\chi^\Ne \norm{\x_k - \x_k} + \sum_{h = 0}^{\Ne - 1} \chi^{\Ne - h - 1} \norm{\tv_k^h} \label{eq:temp-bound}
\end{align}
where (i) holds by triangle inequality, (ii) by contractiveness of $\mathcal{G}_k$, and (iii) by iterating (i) and (ii), and using $\w_k^0 = \hat{\w}_k^0 = \x_k$.

Now, using $\tv_k^\ell \sim \sqrt{2 \gamma} \mathcal{N}(\0_{n N}, \tau^2 \Im_{n N})$ implies $\ev{\norm{\tv_k^\ell}} \leq \tau \sqrt{2 \gamma n N}$ \cite[Lemma~1]{bastianello_distributed_2021}, and taking the expectation of~\eqref{eq:temp-bound} yields
$
    \ev{\norm{\e_k}} \leq \sqrt{5} \tau \sqrt{2 n N \gamma} \frac{1 - \chi^\Ne}{1 - \chi}.
$
Therefore, applying \cref{pr:convergence} the thesis follows. \qed

\bibliographystyle{ieeetr}
\bibliography{references}


\end{document}